\ifdefined\pdfoutput
\pdfoutput=1
\fi

\documentclass[fleqn,11pt]{wlscirep}
\usepackage[utf8]{inputenc}
\usepackage[T1]{fontenc}
\usepackage{multirow}
\usepackage{hhline}
\usepackage{graphicx}

\usepackage{bm}
\usepackage{booktabs}
\usepackage{tabularx}

\usepackage{amsthm}
\usepackage{aliascnt}

\newaliascnt{proposition}{theorem}
\newtheorem{proposition}[proposition]{Proposition}
\aliascntresetthe{proposition}
\newaliascnt{definition}{theorem}

\aliascntresetthe{definition}
\theoremstyle{remark}

\theoremstyle{plain}

\usepackage{algorithm}
\usepackage[
  noEnd=false,
  indLines=true,
  italicComments=false,
  rightComments=true,
  commentColor=black
]{algpseudocodex}
\tikzset{algpxIndentLine/.style={draw=gray!55,very thin}}
\makeatletter
\renewcommand\fs@ruled{%
  \def\@fs@cfont{\bfseries}\let\@fs@capt\floatc@ruled
  \def\@fs@pre{{\color{black}\hrule height.8pt depth0pt}\kern2pt}%
  \def\@fs@post{\kern2pt{\color{black}\hrule}\relax}%
  \def\@fs@mid{\kern2pt{\color{black}\hrule}\kern2pt}%
  \let\@fs@iftopcapt\iftrue}
\makeatother

\usepackage{hyperref}
\usepackage{cleveref}
\usepackage{comment}

\crefname{appendix}{Supplementary Section}{Supplementary Sections}
\Crefname{appendix}{Supplementary Section}{Supplementary Sections}
\crefname{suppfig}{Supplementary Fig.}{Supplementary Figs.}
\Crefname{suppfig}{Supplementary Fig.}{Supplementary Figs.}

\crefname{supptab}{Supplementary Table}{Supplementary Tables}
\Crefname{supptab}{Supplementary Table}{Supplementary Tables}
\crefname{suppsec}{Supplementary Sec.}{Supplementary Secs.}
\Crefname{suppsec}{Supplementary Sec.}{Supplementary Secs.}

\crefname{equation}{eq.}{eq.}
\Crefname{equation}{Eq.}{Eq.}
\crefname{theorem}{Theorem}{Theorems}
\Crefname{Theorem}{Theorem}{Theorems}
\crefname{conjecture}{conj.}{conjs.}
\Crefname{Conjecture}{Conj.}{Conjs.}
\crefname{proposition}{Prop.}{Props.}
\Crefname{proposition}{Prop.}{Props.}
\crefname{definition}{dfn.}{dfn.}
\Crefname{definition}{Dfn.}{Dfn.}
\crefname{remark}{remark}{remark}
\Crefname{Remark}{Remark}{Remark}
\floatname{algorithm}{Supplementary Algorithm}
\crefname{algorithm}{Supplementary Algorithm}{Supplementary Algorithms}
\Crefname{algorithm}{Supplementary Algorithm}{Supplementary Algorithms}
\crefname{section}{Sec.}{Secs.}
\Crefname{section}{Sec.}{Secs.}
\crefname{equation}{Eq.}{Eqs.}
\Crefname{equation}{Eq.}{Eqs.}
\crefname{figure}{Fig.}{Figs.}
\Crefname{figure}{Fig.}{Figs.}
\crefname{table}{Tab.}{Tabs.}
\Crefname{table}{Tab.}{Tabs.}
\crefname{thm}{Theorem}{Theorems}
\crefname{thm}{Theorem}{Theorems}
\crefname{conj}{Conj.}{Conjs.}
\Crefname{conj}{Conj.}{Conjs.}
\crefname{dfn}{Dfn.}{Dfns.}
\crefname{dfn}{Dfn.}{Dfns.}
\crefname{remark}{Remark}{Remarks}
\Crefname{remark}{Remark}{Remarks}
\crefname{prop}{Prop.}{Prop.}
\Crefname{prop}{Prop.}{Prop.}
\Crefname{algorithm}{Supplementary Algorithm}{Supplementary Algorithms}

\usepackage{threeparttable}  %
\usepackage{dsfont}

\usepackage[table]{xcolor} %

\usepackage{xspace} %

\let\rk\relax %
\newcommand{\rk}{\operatorname{rk}} %
\newcommand*\diff{\mathop{}\!\mathrm{d}}

\usepackage{pifont} %

\newcommand \spm[1] {\footnotesize{$\pm$#1} }

\usepackage[normalem]{ulem}%

\title{Higher-Order Molecular Grammars for Generative and Foundation Models in Chemistry}

\author[1*]{Yiming Huang}
\author[2]{Yujie Zeng}
\author[3]{Vijay Prakash Dwivedi}
\author[1]{Simone Foti}
\author[4*]{Jianmin Wang}
\author[3]{Jure Leskovec}
\author[1*]{Tolga Birdal}
\affil[1]{Department of Computing, Imperial College London, London, UK}
\affil[2]{School of Electronic Engineering and Computer Science, Queen Mary University of London, London, UK}
\affil[3]{Department of Computer Science, Stanford University, Stanford, CA, USA}
\affil[4]{Department of Integrative Biotechnology, Yonsei University, Incheon, Korea}
\affil[*]{Emails: y.huang24@imperial.ac.uk, jmwang@yonsei.ac.kr, t.birdal@imperial.ac.uk}

\begin{abstract}
Molecular learning models are strongly shaped by their underlying representations.
Yet standard sequential and graph formalisms struggle to explicitly encode higher-order topology, such as ring systems and recurring motifs.
Existing higher-order representations can capture these structures directly, but they are often computationally demanding and difficult to decode into valid molecules.
Here, we introduce Higher-order Grammar Representation (HGR), a principled, topology-aware framework that lifts molecules to combinatorial complexes and parses each complex into a compact sequence of production rules under a context-free higher-order grammar.
By serialising higher-order topology into rule sequences, HGR makes these structures directly compatible with standard sequence models, avoiding the computational overhead of explicit higher-order encodings while preserving topological expressiveness.
To reduce benchmark bias towards simple ring systems, we construct RingDiv, a ring-enriched benchmark containing 1.18 million molecules, including the curated RingDiv300k subset, and introduce the ring diversity index (RDI) to quantify ring-system coverage.
In molecular generation, HGR-based models uniquely combine 100\% validity by construction with leading distributional alignment, ranking first in FCD on all five generation benchmarks.
In representation learning, HGR-FM achieves the highest mean AUC across seven MoleculeNet benchmarks under both transfer protocols, improving on the strongest baseline by 8.3 and 3.3 AUC points under probing and full fine-tuning, respectively.
Collectively, these results establish HGR as an efficient higher-order representation for molecular generation and transferable representation learning.
\end{abstract}

\begin{document}

\flushbottom
\maketitle
\thispagestyle{empty}

\section*{Main}
\pdfbookmark[1]{Main}{sec:main}

Molecular representation is a cornerstone of computational chemistry and drug design, bridging the gap between chemical structures and their biological, chemical, or physical properties~\cite{t-SMILES-NC2024}.
Effective representations enable efficient and precise navigation of chemical space and thus are essential for drug discovery tasks, including virtual screening, activity prediction, and scaffold hopping~\cite{du2024machine}.
However, current molecular representations struggle to encode higher-order structural information, that is, relations involving more than two atoms or molecular substructures, for generation and other tasks~\cite{wan2025multi,rottach2025topology}.
This limitation matters because many chemically important structures, including rings and functional groups, are intrinsically higher-order~\cite{HigherOrderReview2020}.
Ring systems constrain molecular shape and position substituents~\cite{driggers2008exploration}, whereas functional groups establish interaction sites through polarity, hydrogen bonding and acidity.
Together, these higher-order units connect atom-level connectivity to meso-scale chemical organisation and, ultimately, molecule-level behaviour, including reactivity, metabolic stability and interactions with biological targets~\cite{taylor2014rings,RingSystems-EJMC2025}.
Our statistical analysis further shows that ring-system usage is highly concentrated, with a small set recurring across many drug-like molecules. This concentration makes ring-system diversity tractable to model.
Beyond chemistry, higher-order relations are increasingly studied across diverse contexts, including the neural population activity underlying spatial navigation \cite{gardner2022toroidal}, group hunting behaviours in ecological competitive networks \cite{grilli2017higher}, and recommendations from multiple friends in social networks \cite{centola2010spread}.
Nevertheless, higher-order topology has often been encoded only implicitly, reduced to pairwise graphs, or treated as isolated motifs~\cite{HigherOrderReview2020}.

Formally, molecular representation refers to the process of encoding molecules into mathematical or computational formats that algorithms can process to model, analyse, and predict molecular behaviour.
Broadly, representations fall into three categories: early atom-level linear representations, graph-based representations, and the emerging higher-order topological representations (\cref{fig:HGR_frame}\textbf{b}). 
In addition, there exist various non-invertible mappings, such as molecular fingerprints \cite{morgan_fp2010}, which transform molecules into computationally convenient forms (e.g., vectors in feature space) but generally lack strict reversibility.

Simplified Molecular-Input Line-Entry System (SMILES) \cite{weininger1988smiles} remains one of the most widely used molecular representations, using parentheses and numbers to denote branching points and ring closures.
Although compact and interpretable, SMILES is prone to generating invalid molecules due to inherent syntactic and semantic flaws.
To mitigate these challenges, Self-Referencing Embedded Strings (SELFIES)~\cite{SELFIES-2020} improve robustness and validity in deep generative models, albeit at the cost of simplicity, interpretability and compactness.
With the rapid rise of large language models, sequence-based representations have regained attention in molecular modelling and drug discovery~\cite{LLM4SD-NMI2025,ESM3-Science2025}.
However, such string-based representations still do not explicitly encode hierarchical or meso-scale molecular structure.
Molecular graphs provide a direct representation of atoms and bonds, and graph neural networks (GNNs) have been widely used for molecular learning and generation~\cite{MolCLR-NMI2022,HiGCN2024}.
Despite these advances, GNNs face several limitations. Standard message-passing GNNs are bounded by the Weisfeiler–Lehman (WL) test~\cite{GIN-ICLR2019}; deeper architectures can suffer from over-smoothing~\cite{oversmoothing} and over-squashing~\cite{GNN-bottleneck-ICLR2021}, which hinder long-range dependency modelling; and simple non-learnable readout functions can constrain graph-level representations.
More fundamentally, both sequence- and graph-based representations leave hierarchical higher-order topology implicit, limiting their ability to transfer across structure-sensitive tasks~\cite{TDL-position+ICML2024}.

Recently, higher-order structures have gained increasing attention in molecular modelling because they expose multi-way molecular organisation in forms useful for interpretability, efficient modelling and synthetic-accessibility analysis~\cite{HigherOrderReview2020,hogdiff-ICLR2026}.
A natural entry point to higher-order molecular topology is the ring system. %
Rings provide basic molecular shape, tune rigidity and substituent geometry, and strongly influence global physicochemical and electronic properties. 
Improved synthetic methods and building-block availability further broaden accessible ring chemical space, often leveraging conformational constraints to enhance potency and selectivity~\cite{taylor2014rings,shearer2022rings}.
Indeed, only a small fraction of bioactive molecules lack rings, and the vast majority of FDA-approved small-molecule drugs contain at least one ring~\cite{RingSystems-EJMC2025}.
Moreover, molecules are inherently hierarchical~\cite{wan2025multi}. The atom level encodes fine-grained local connectivity and semantic similarity. The meso-scale captures chemically meaningful substructures, such as rings, functional groups, linkers and pharmacophores, that act as property-determining ``semantic units''. These units govern polarity, hydrogen bonding, acidity, reactivity and conformational rigidity, linking local chemistry to global topology. The molecule level then summarises the overall structural patterns.
If a model relies only on single-level reconstruction or contrastive supervision, it can either overemphasize local validity while producing globally incoherent assemblies (e.g., implausible motif fusion/attachment) or preserve coarse global structure at the expense of chemically faithful local details, making it difficult to maintain local consistency and global stability simultaneously.

Motif-based representations provide one pragmatic route to capture such higher-order topology. They can be interpreted within a hypergraph framework, where each fragment is treated as a hyperedge connecting multiple atoms~\cite{erlanson2016twenty}.
This shift from atom- to fragment-level units parallels the transition from character- to word-level representations in natural language processing, and can yield more compact and expressive molecular encodings.
These approaches broadly fall into fragment-string and fragment-graph representations.
Fragment-string methods represent molecules as sequences of substructure units, including SMILES Pair Encoding (SPE)~\cite{SPE-2021}, t-SMILES~\cite{t-SMILES-NC2024}, Group SELFIES~\cite{GroupSELFIES-2023}, and fragSMILES~\cite{fragSMILES-CC2025}.
These methods preserve the convenience of string-based modelling by serialising substructures into a one-dimensional order, but relations among fragments remain implicit in the string syntax.
Graph-based fragment methods instead operate on fragment graphs, including motif-assembly generators and graph grammars.
For example, JT-VAE~\cite{JunctionTree+ICML2018} first generates a tree and then expands selected nodes by attaching substructures from a learned vocabulary.
HierVAE~\cite{HierVAE-ICML2020} follows the same vocabulary-based principle but uses a multiresolution representation to improve combinatorial attachment and generate larger molecules.
Graph grammars~\cite{MHG-VAE-ICML2019,DEG+ICLR2022,RWG+ICML2024,FMG-ICML2025} represent graphs through derivations over reusable subgraph rules, making them the closest existing family to HGR among fragment-based methods.
However, graph grammars often face a trade-off between rule expressiveness and reuse. Context-specific fragments yield sparse, long-tailed rule vocabularies, whereas generic fragments weaken chemically meaningful motif semantics.
Taken together, these fragment-based representations capture multi-atom interactions by treating fragments as higher-order units, but they do not explicitly encode the hierarchical part-whole organisation linking atoms, motifs, ring systems and larger molecular structures.

Topological deep learning (TDL) offers a general route beyond sequence-, graph- and fragment-based representations by modelling multi-way relations with hypergraphs, simplicial complexes, cell complexes and combinatorial complexes~\cite{TDL-position+ICML2024,CCs-TDL-2022}.
Such structures are typically encoded using families of incidence (boundary) matrices. Although this makes higher-order structure explicit, generating over such structures is difficult because the incidence relations across topological dimensions must remain mutually consistent.
For example, a face generated in an edge--face incidence matrix is structurally inconsistent if its boundary edges are absent from the node--edge incidence matrix.
Molecular generation with incidence-based representations must navigate both the vast chemical space and a space of possible higher-order incidence relations that grows combinatorially with molecular size.
Together with the limitations of sequence-, graph- and fragment-based representations discussed above, these challenges reveal a common trade-off between explicit higher-order structure and computational tractability. 
Sequence and graph encodings leave higher-order topology implicit, fragment-based methods do not explicitly encode its hierarchy, and incidence-based encodings make scalable generation difficult. 
Resolving this trade-off calls for an expressive and lossless representation that remains computationally tractable for both molecular generation and property prediction.

In this work, we propose \textbf{Higher-order Grammar Representation (HGR)}, a principled and lossless representation that converts hierarchical higher-order molecular topology into a compact sequence of production rules.
The central idea of HGR is to \emph{(i)} lift a molecular graph into a combinatorial complex whose cells correspond to hierarchical and higher-order (multi-atom) relations, then \emph{(ii)} parse each molecular complex into a sequence of production rules via a bottom-up contraction procedure, thereby inducing a higher-order grammar that supports exact decoding.
Intuitively, HGR treats molecules as assemblies of reusable higher-order structural ``words'', while retaining the attachment information needed for exact reconstruction.
Each lifted cell contributes one production-rule token, and the resulting ordered sequence preserves the complex's hierarchy across atom, motif and molecule scales.
This sequence formulation makes HGR compatible with mature sequence architectures while avoiding the combinatorial cost of generation over higher-order incidence matrices. 
To address rule explosion, we introduce two lifting schemes, Motif-Induced Grammar (MIG) and Ring Scaffold Grammar (RSG), which respectively simplify complex anchor interfaces and stabilise lifting around ring systems.
Evaluating higher-order molecular representations also requires benchmarks that adequately cover complex ring systems, which are underrepresented in commonly used molecular generation datasets. 
We therefore construct RingDiv, a ring-enriched benchmark of 1.18 million molecules, together with its topology-balanced RingDiv300k subset, to evaluate molecular generation in ring-diverse chemical space. 

In summary, we make five contributions:
\emph{(i) Lifting.} We introduce two data-driven lifting strategies that convert a molecular graph into a combinatorial complex while controlling boundary complexity and rule explosion.
\emph{(ii) Higher-order grammar.} We propose HGR, a grammar-induction framework that serialises a molecule's hierarchical higher-order topology into a compact HGR string and enables exact reconstruction through deterministic reverse derivation.
\emph{(iii) Structure-aware rule encoding.}
We initialise production-rule tokens with permutation-invariant structural descriptors, yielding rule embeddings that reflect topological similarity and generalise to unseen rules.
\emph{(iv) Generative and transfer applications.} 
We instantiate HGR in two complementary regimes: grammar-constrained molecular generation through VAE-prior sampling (HGR-VAE) or latent diffusion sampling (HGR-LDF), and foundation modelling (HGR-FM) for transferable representation learning. 
Both regimes operate on the same HGR representation, enabling reuse across generative and predictive tasks.
\emph{(v) RingDiv benchmark.} We construct RingDiv, a ring-enriched benchmark that mitigates the ring-system bias of existing datasets.
Our evaluation yields three key insights:
\emph{(i)} the grammar-based generators are valid by construction and achieve strong chemical and graph-structural distributional alignment;
\emph{(ii)} the same representation is broadly transferable, letting a single pretrained model support multiple downstream tasks and improve property prediction; and
\emph{(iii)} ablations attribute these gains to the explicit encoding of hierarchical higher-order topology rather than to sequence modelling alone, identifying it as a useful inductive bias.
Relative to our preliminary report \cite{hgr-workshop}, the present study substantially extends the generative framework by refining HGR-VAE and introducing latent diffusion sampling, together with a fuller methodological specification and broader evaluation across five molecular-generation benchmarks. 
We further provide a more detailed characterisation of RingDiv, expanded ablation studies, and analyses of chemical and topological generation fidelity.
\section*{Results}
\pdfbookmark[1]{Results}{sec:Results}

\subsection*{Higher-order grammar representation}
\pdfbookmark[2]{Higher-order grammar representation}{sec3.1}
    
Drug discovery requires efficient navigation of chemical space, where the number of drug-like molecules is estimated at $10^{23}$--$10^{60}$, far beyond the reach of experimental and virtual screening~\cite{polishchuk2013estimation}.
By contrast, our statistical analysis shows that ring-system usage is highly concentrated: only 630 distinct ring systems are observed in DrugBank~\cite{knox2024drugbank} (\cref{fig:zinc250k_grammar_statistics}\textbf{a}).
More broadly, chemically meaningful motifs such as ring systems and functional groups encode multi-way relations that pairwise edges alone cannot fully capture.
We therefore represent each molecule with HGR (\cref{fig:HGR_frame}): starting from the molecular graph (\cref{fig:HGR_frame}\textbf{a}), we lift it into a combinatorial complex (CC, Fig. \ref{fig:HGR_frame}\textbf{c})~\cite{CCs-TDL-2022}, a formalism that unifies graphs, hypergraphs, and simplicial and cell complexes, and serialise the complex into a compact string of reusable production-rule tokens (\cref{fig:HGR_frame}\textbf{e}).
In this representation, cells encode multi-way relations and carry hierarchical ranks. 
The resulting HGR string preserves hierarchy and attachment information for exact decoding, making higher-order topology tractable for standard sequence models.

\begin{figure}
    \centering
    \includegraphics[width=1\linewidth]{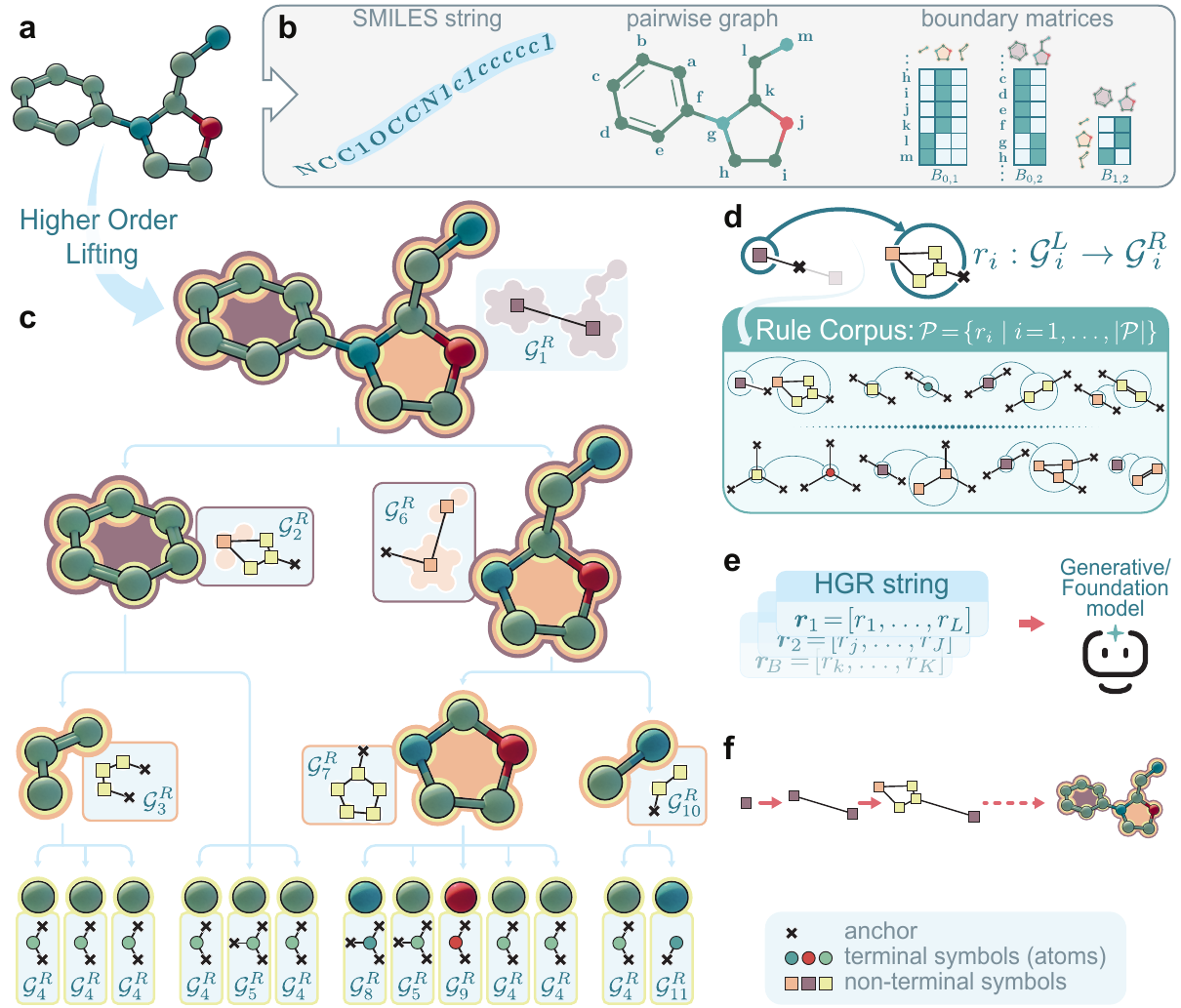}
\caption{\textbf{Overview of molecular representations and our HGR framework.}
{\bf a}, Example molecule used throughout the schematic.
{\bf b}, Representative molecular representations, illustrated by a SMILES string, a pairwise atom--bond graph and higher-order boundary matrices such as $B_{0,1}$, $B_{0,2}$ and $B_{1,2}$, which encode incidence relations.
\textbf{c}, Higher-order lifting produces a combinatorial complex that encodes the molecule's hierarchical higher-order topology, with its cells forming a contraction hierarchy. 
For clarity, each cell is annotated only with the right-hand-side graph $\mathcal{G}^R_i$ of its associated rule.
\textbf{d}, Each production rule $r_i:\mathcal{G}^L_i\rightarrow\mathcal{G}^R_i$ maps a left-hand-side interface graph that specifies a non-terminal site and its ordered anchor interface to the corresponding right-hand-side local graph.
Distinct rules induced across the molecular library form the shared corpus $\mathcal{P}=\{r_i\mid i=1,\ldots,|\mathcal{P}|\}$.
{\bf e}, HGR serialises each complex as an ordered string $\bm{r}$ of reusable rule tokens and thereby provides a common input representation for diverse molecular-learning models, such as generative and foundation models.
{\bf f}, Starting from the start symbol $S$, successive rule applications replace selected non-terminal sites with their corresponding $\mathcal{G}^R_i$ and reconnect boundary bonds through the anchor interfaces, thereby reconstructing the combinatorial complex and its associated molecular graph.
}
\label{fig:HGR_frame}
\end{figure}

\paragraph{Lifting operation.}
In topological deep learning, the conversion from a graph to a higher-order domain is known as lifting.
We adopt data-driven lifting based on a motif vocabulary mined from the molecule library, where motifs are frequent recurrent subgraphs in the dataset.
To ensure invertibility, each production rule records the lifted motif together with an ordered anchor interface, a fixed ordering of its external neighbours and boundary bonds that is retained when the motif is contracted.
This interface fixes the boundary correspondence during decoding and thereby supports exact reconstruction of the lifted complex and its associated molecular graph.
However, the rule corpus may expand when motifs with many anchors admit multiple ordered anchor interfaces, each potentially inducing a separate rule.

To balance expressivity and rule compactness, we instantiate two lifting schemes named by their induced grammars: Motif-Induced Grammar (MIG) and Ring Scaffold Grammar (RSG).
Each variant $\eta\in\{\mathrm{MIG},\mathrm{RSG}\}$ maps a molecular graph to a combinatorial complex $\mathcal{C}_{\eta}$ and induces the corresponding grammar $\Gamma_{\eta}$.
They share a frequency-guided motif-merging step but differ in the subsequent constraints: MIG lifting adds an anchor-resolution stage that curbs anchor proliferation via best-first motif expansion, whereas RSG lifting further stabilises lifting for cyclic molecules through ring-system-aware merging, mitigating rule inflation from complex ring attachments.
Formal definitions and implementation details are provided in Methods.
Unless otherwise stated, the main experiments use RSG, which provided a favourable balance between construction efficiency, rule compactness and empirical performance in our comparisons. MIG is reported as a supplementary ablation.
The lifted complex can be viewed as a hierarchical tree (\cref{fig:HGR_frame}{\bf c}): leaves correspond to atomic micro-structures, whereas internal nodes capture higher-order meso-structures.

\paragraph{Higher-order grammar induction and derivation.}
HGR serialises a lifted complex and its associated molecular graph as a compact sequence of reusable production-rule tokens.
The sequence records a bottom-up contraction in which every child subtree precedes its parent, and the resulting derivation hierarchy explicitly describes how atoms form motifs and how these motifs assemble into the complete molecule. 
During parsing, each step contracts one cell into a fresh non-terminal symbol and records the corresponding expansion as a production rule with an ordered interface graph and a local structural pattern (\cref{fig:HGR_frame}\textbf{d}).
Parsing continues until the full structure is contracted to a single root non-terminal, which is designated as the start symbol $S$.
Across the molecular library, structurally identical contractions are assigned the same production rule, so induction collects only the distinct rules in a shared corpus (\cref{fig:HGR_frame}\textbf{d}).
Applying the recorded rules in reverse losslessly reconstructs the original complex and associated molecular graph (\cref{fig:HGR_frame}\textbf{f}). 
Formal definitions and details are provided in Methods.

Whereas the shift from sequential representations to graph representations has advanced molecular machine learning, HGR takes a further leap by explicitly incorporating hierarchical higher-order structures.
This rank-aware construction lets a single HGR string span atoms, motifs and ring systems.
HGR offers three practical advantages: production rules encode reusable and interpretable higher-order motifs; grammar induction from real molecular complexes preserves valency and bonding constraints, ensuring chemical validity; and the representation supports exact reconstruction and assembly of new molecules.

\paragraph{Grammar compactness and rule reuse.}
We next quantify how the lifting variants shape grammar complexity and rule reuse on ZINC250k~\cite{data:zinc250k}, with additional results provided in \cref{SI-fig:qm9_grammar_statistics}.
As more molecules are added during induction, the corpus size (unique production rules) grows sublinearly with dataset size and approaches saturation, with MIG and RSG converging to substantially more compact corpora than the Molecular Hypergraph Grammar (MHG) baseline~\cite{MHG-VAE-ICML2019} (\cref{fig:zinc250k_grammar_statistics}{\bf b}).
Consistently, MIG and RSG exhibit markedly more balanced rule-frequency distributions, attenuating the long tail of rarely used rules and increasing overall rule reuse (\cref{fig:zinc250k_grammar_statistics}{\bf c}).

\subsection*{RingDiv: a ring-diversity benchmark}
\pdfbookmark[2]{RingDiv: a ring-diversity benchmark}{sec3.3}

\begin{figure}
\centering
\includegraphics[width=\linewidth]{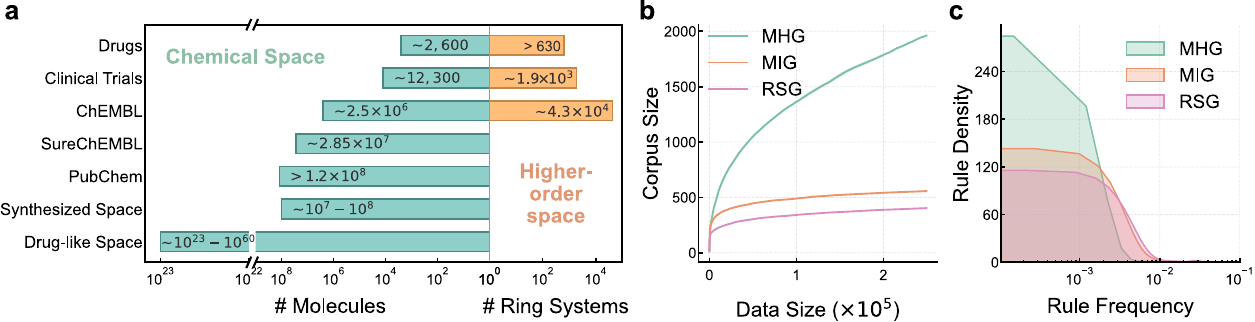}
\caption{\textbf{Ring-system sparsity and grammar-rule statistics.}
{\bf a}, Scale of experimentally accessible chemical collections (left axis) compared with ring-system diversity observed in practice (right axis).
{\bf b}, Growth of the grammar corpus on ZINC250k, quantified as the number of unique rules discovered as molecules are progressively added during grammar induction.
{\bf c}, Frequency distributions of grammar rules induced on ZINC250k.
}
\label{fig:zinc250k_grammar_statistics}
\end{figure}

Beyond representation design, rigorous evaluation requires benchmarks that cover ring-rich chemical space.
Conventional molecular generation benchmarks underrepresent ring-system diversity, creating a mismatch with modern drug-discovery chemistry, where fused, bridged and multi-ring architectures are increasingly relevant~\cite{RingSystems-EJMC2025}.
This mismatch can distort evaluation by favouring models that fit simplified benchmark distributions while leaving their behaviour in ring-diverse chemical space unresolved.
Results on QM9 and MOSES (\cref{SI-sec:gen-rel}) further show that modern generators already approach empirical reference levels on conventional benchmarks with simpler ring-system profiles, leaving little headroom to stress-test models or identify practical routes for improvement.

We therefore constructed RingDiv, a benchmark of 1.18 million molecules explicitly enriched for diverse ring systems.
RingDiv300k was then derived from RingDiv as a topology-balanced development subset of 299,819 molecules.
The construction began with approximately 143 million compounds from PubChem, SureChEMBL and five commercial screening libraries. After deduplication, these compounds were merged into a candidate corpus.
Molecules in this corpus were then annotated by ring size and fusion pattern, with macrocycles excluded to keep the benchmark focused on drug-like ring diversity. 
We then applied quota-first sampling to enrich rare but informative ring topologies (e.g., medium rings, spirocyclic and bridged systems), followed by MaxMin selection over Morgan fingerprints to preserve overall structural diversity and suppress near-duplicates (see Methods and \cref{SI-sec:ringdiv-sampling}).

To quantify ring-system coverage relative to standard generation benchmarks, we introduce the Ring Diversity Index (RDI), a ratio-based measure of ring-system coverage per molecule:
\begin{equation}
    \text{RDI} = \frac{\#\,\text{Unique Ring Systems}}{\#\,\text{Molecules}} .
\end{equation}
As shown in \cref{fig:ringdiv-analysis}\textbf{a}, RingDiv300k and RingDiv exhibit higher RDI than most benchmarks.
\Cref{fig:ringdiv-analysis}\textbf{b} complements the ratio-based RDI with absolute counts of unique ring systems, showing broader ring-topology coverage in the RingDiv variants.

Because RDI captures dataset-level coverage rather than per-molecule complexity, we complemented it with the Quantitative Ring Complexity Index (QRCI)~\cite{QRCI-2025} and separately examined ring-category prevalence and the unique-scaffold ratio.
Higher QRCI values indicate greater ring-system complexity and topological diversity. 
The QRCI distributions of RingDiv300k and RingDiv are concentrated towards the upper end of the approved-drug range (QRCI = 1.8--5), indicating enrichment for greater ring complexity without a shift into the overcomplex regime.
RingDiv300k and RingDiv show higher prevalences of medium rings and spirocyclic, fused and bridged ring systems than most benchmarks (\cref{fig:ringdiv-analysis}\textbf{d}). They also have the two highest unique-scaffold ratios.
Across the evaluated benchmarks, RingDiv and RingDiv300k combine broad ring-system coverage with relatively high per-molecule ring complexity and enrichment of uncommon ring topologies.

\begin{figure}
\centering
\includegraphics[width=\linewidth]{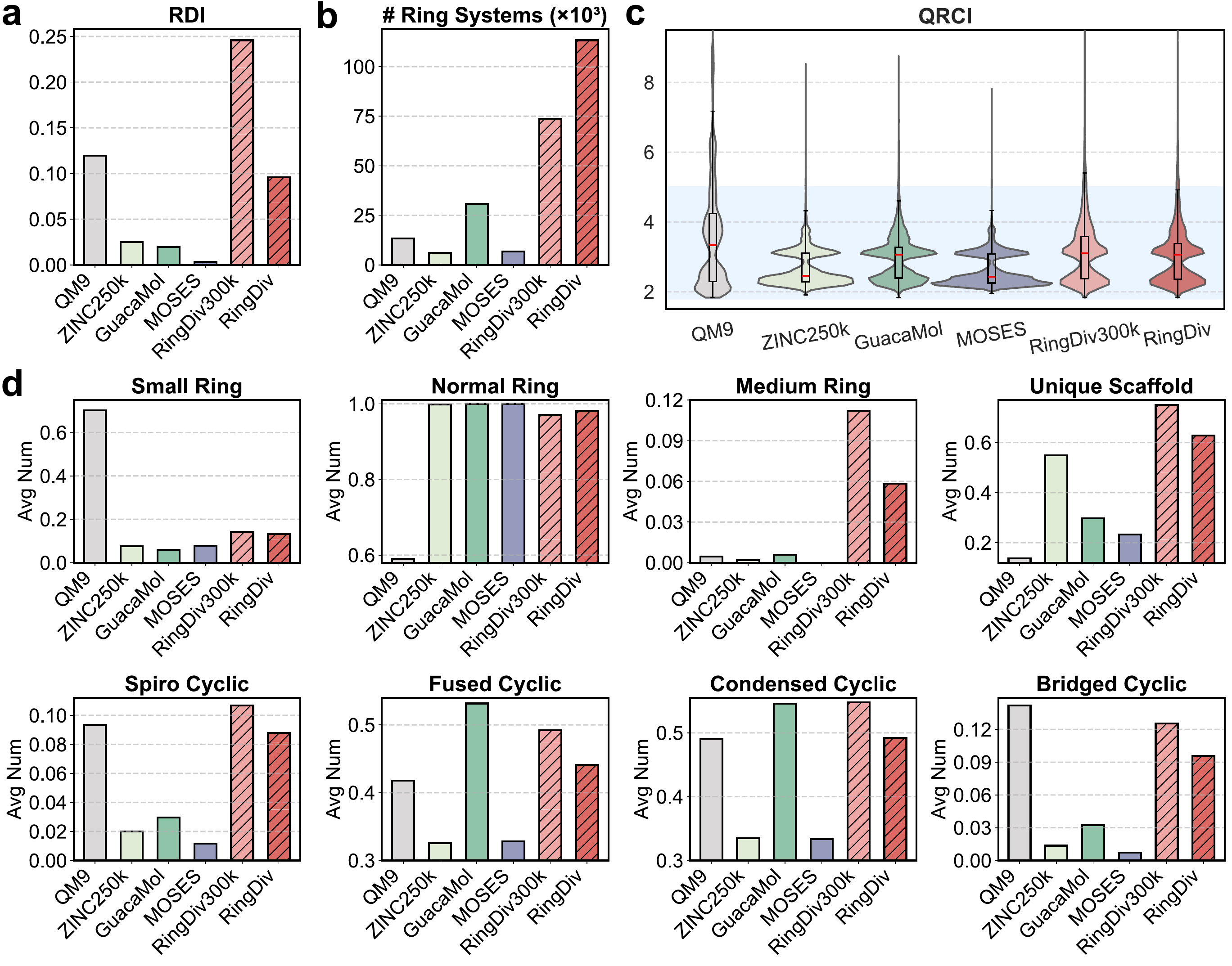}
\caption{{\bf Ring-system diversity across molecular benchmarks.}
{\bf a}, Ring Diversity Index (RDI).
{\bf b}, Number of unique ring systems.
{\bf c}, Quantitative Ring Complexity Index (QRCI) distributions. Higher QRCI values indicate greater ring complexity and topological diversity. The light-blue region (QRCI = 1.8--5) denotes the range typically observed for approved drugs, balancing structural diversity and synthetic feasibility. Values above this range indicate more complex ring systems that may reduce drug-likeness and synthetic accessibility.
{\bf d}, Fractions of molecules containing each ring-size or ring-connectivity category, together with the unique-scaffold ratio.
}
\label{fig:ringdiv-analysis}
\end{figure}

\subsection*{Molecular generation with HGR}
\pdfbookmark[2]{Molecular generation with HGR}{sec3.4}

Equipped with the HGR representation and the ring-enriched RingDiv benchmark, we first turn to molecular generation, a core drug-discovery task that supplies novel candidates and shapes downstream screening quality~\cite{du2024machine}.
However, current molecular generative models predominantly rely on sequence- or graph-based representations that emphasize local features such as atoms and bonds but fail to capture intrinsic higher-order molecular topology.
Despite recent advances in topology-aware generation~\cite{hogdiff-ICLR2026}, direct generation of higher-order structures remains elusive.
The difficulty stems from the combinatorial higher-order search space and the need to maintain consistency across incidence relations~\cite{TDL-position+ICML2024}. 

To close this gap, we represent molecules as expressive HGR strings, whose rule tokens make higher-order topology explicit.
This recasts molecular generation as sequence modelling,  allowing standard sequence generative models to operate on higher-order molecular structure.
Specifically, we adopt a Grammar Variational Autoencoder (GVAE)~\cite{GVAE+ICML2017} tailored to the induced grammar, whose decoder restricts generation to admissible HGR strings while its latent space captures semantically meaningful variations in molecular topology.
Following training, molecular generation proceeds via two parallel routes:
\emph{(i)} direct sampling from the latent space to obtain latent codes (HGR-VAE), or
\emph{(ii)} latent-space diffusion to synthesize latent codes (HGR-LDF).
The grammar decoder then maps the resulting latent codes to molecular structures.
HGR-based generation produces molecular combinatorial complexes that retain hierarchical higher-order topology, together with associated molecular graphs recovered from the recorded rule interfaces.
Because the production rules are derived from subgraphs of real molecules, rule application preserves recorded bonds and ordered anchor interfaces, yielding valency-preserving decoding by construction.
For each generated molecule, the corresponding rule sequence records how reusable higher-order rules assemble the final structure.
The representation therefore preserves local chemical constraints while recording higher-order structure and its assembly trajectory.

\begin{table}[!t]
\centering
\begin{threeparttable}
\caption{\textbf{Generation performance across molecular benchmarks.}}
\label{tab:main_gen3}
\begin{tabular}{clccccc}
\toprule
\multirow{7}{*}{\rotatebox[origin=c]{90}{GuacaMol}}
& Method & Val. $\uparrow$ & Unique. $\uparrow$ & Nov. $\uparrow$ & FCD score $\uparrow$ & KL score $\uparrow$ \\
\cmidrule(lr){2-7}
& DiGress~\cite{DiGress+ICLR2023} & 85.2 & 100.0 & 99.9 & 68.0 & 92.9 \\
& DeFoG~\cite{DeFoG-ICML2025} & 99.0 & 99.0 & -- & 73.8 & 97.7 \\
& HoG-Diff~\cite{hogdiff-ICLR2026} & 99.1 & 100.0 & 97.0 & 78.3 & 96.9 \\
& SELFIES-VAE~\cite{SELFIES-2020} & 100.0 & 100.0 & 99.7 & 39.9 & 88.0 \\
& HGR-VAE & 100.0 & 100.0 & 99.4 & 83.7 & 93.8 \\
& HGR-LDF & 100.0 & 100.0 & 98.9 & 84.4 & 95.7 \\
\midrule
\multirow{7}{*}{\rotatebox[origin=c]{90}{RingDiv300k}}
& Method & Val. $\uparrow$ & Unique. $\uparrow$ & Nov. $\uparrow$ & FCD $\downarrow$ & KL score $\uparrow$ \\
\cmidrule(lr){2-7}
& DiGress~\cite{DiGress+ICLR2023} & 98.9 & 100.0 & 100.0 & 0.6300 & 98.2 \\
& DeFoG~\cite{DeFoG-ICML2025} & 98.3 & 100.0 & 100.0 & 0.8575 & 98.7 \\
& HoG-Diff~\cite{hogdiff-ICLR2026} & 100.0 & 99.9 & 99.6 & 0.9250 & 98.2 \\
& SELFIES-VAE~\cite{SELFIES-2020} & 100.0 & 100.0 & 100.0 & 2.4266 & 92.8 \\
& HGR-VAE & 100.0 & 100.0 & 100.0 & 0.5882 & 98.1 \\
& HGR-LDF & 100.0 & 100.0 & 100.0 & 0.4652 & 99.2 \\
\bottomrule
\end{tabular}
\begin{tablenotes}
\footnotesize
\item Results from our runs are three-seed means; published values are shown as reported. Extended results, including mean $\pm$ s.d. for our runs, are provided in \cref{SI-sec:gen-rel}.
FCD is reported as a distance for RingDiv300k (lower is better), whereas GuacaMol reports the benchmark FCD score $=100\exp(-0.2 \times \mathrm{FCD})$ (higher is better).
\end{tablenotes}
\end{threeparttable}
\end{table}

To assess generation performance, we compared HGR-based generators with representative diffusion-based, higher-order-aware and sequence-control baselines on GuacaMol and RingDiv300k, two larger and more discriminative benchmarks (\cref{tab:main_gen3}).
SELFIES-VAE~\cite{SELFIES-2020} serves as a syntax-valid sequence representation control in the main comparison.
Consistent with our valency-preserving design, both HGR-VAE and HGR-LDF achieved 100\% RDKit validity on all five generation benchmarks.
HGR-LDF showed strong distributional alignment, achieving the best FCD score on GuacaMol (84.4 versus 78.3 for the strongest diffusion baseline) and the lowest FCD and best KL score on RingDiv300k (0.4652 and 99.2). HGR-VAE ranked second by FCD on both benchmarks.
SELFIES-VAE also achieved 100\% validity but showed weaker distributional alignment on RingDiv300k (FCD 2.4266; KL score 92.8).
HGR is thus the only representation that combines full validity with leading chemical-distribution alignment.
One exception in the main comparison is the GuacaMol KL score, where HGR-LDF trails the strongest diffusion baselines.

Extended results in \cref{SI-sec:gen-rel} support this interpretation across additional datasets, baselines and structural metrics.
Among these structural checks, NSPDK MMD is used as a secondary structural-distribution metric, and curvature-filtration descriptors are reported as auxiliary metrics.
SMILES-VAE reached competitive distributional scores only at substantially lower validity, highlighting the need for a validity-preserving representation.
HGR-based generators also achieved lower FCD and NSPDK MMD than MHG-VAE, the closest molecular hypergraph grammar baseline.
Within HGR, RSG generally improved topology-distribution alignment over MIG, indicating that ring-system-aware grammar construction contributes to generation quality.
These extended benchmarks further show that QM9 and MOSES are close to saturation, underscoring the need for more discriminative benchmarks like RingDiv.
Beyond generation quality, HGR-VAE trained fastest on all five benchmarks and sampled in minutes rather than the hours required by diffusion baselines (\cref{SI-tab:gen_runtime}).

\begin{figure}
\centering
\includegraphics[width=1\linewidth]{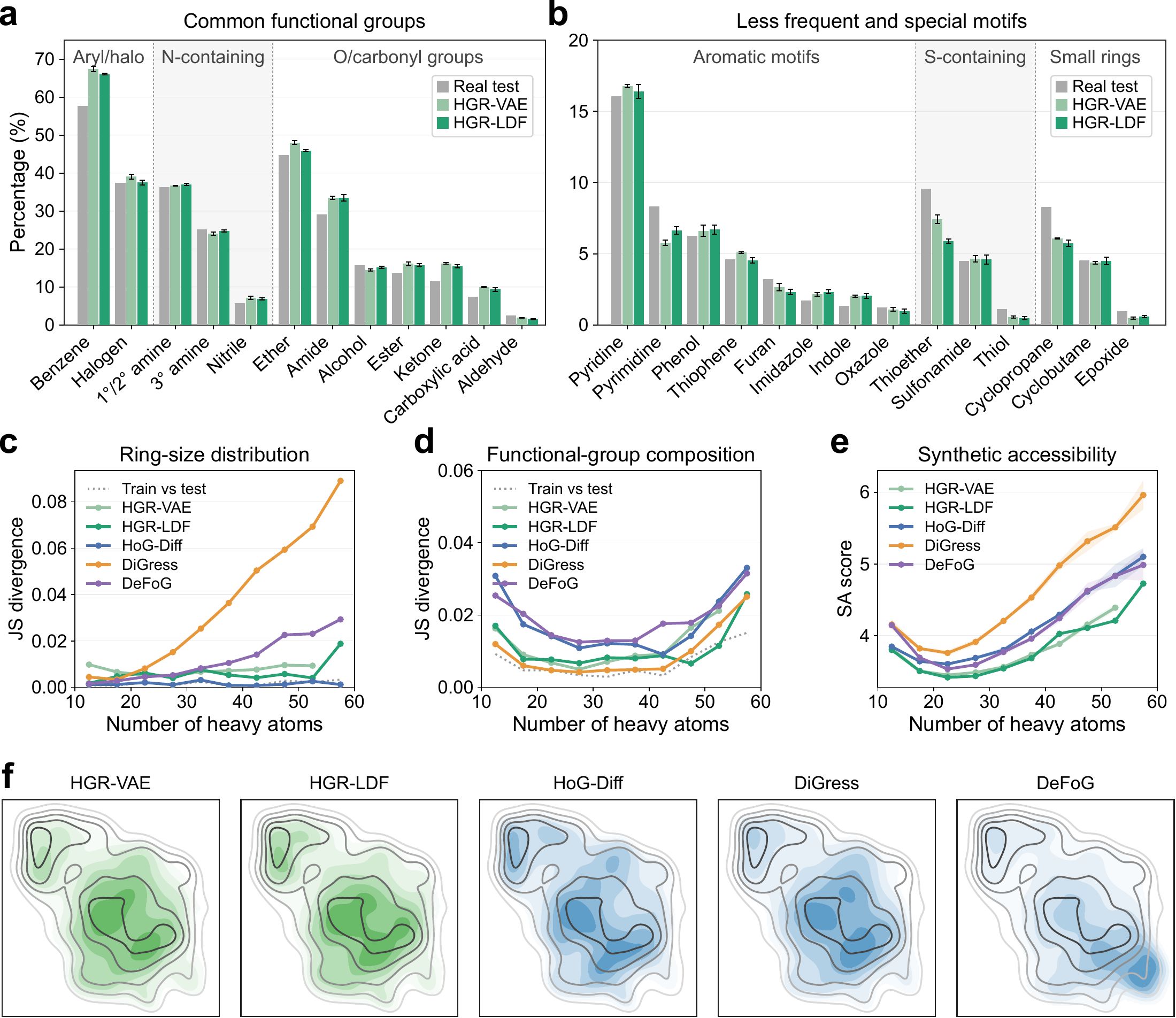}
\caption{{\bf Generation fidelity and chemical-space coverage of molecular generative models on RingDiv300k.}
{\bf a}, {\bf b}, Functional-group prevalence in the RingDiv300k test set and HGR-generated molecules for common functional groups (\textbf{a}) and less frequent or special motifs (\textbf{b}). Bars show motif prevalence; error bars denote s.d. across three generation seeds.
{\bf c}, {\bf d}, Size-resolved Jensen--Shannon (JS) divergence between generated and test-set distributions for ring-size distributions ({\bf c}) and functional-group composition ({\bf d}), with molecules binned by heavy-atom count. Lower JS indicates closer agreement with the test distribution; grey dotted curves show the train-versus-test reference.
{\bf e}, Lines show size-resolved synthetic-accessibility (SA) scores; shaded bands denote s.d. across generation seeds. Lower SA indicates easier synthesis.
{\bf f}, Chemical-space coverage in Morgan-fingerprint Uniform Manifold Approximation and Projection (UMAP) space.
Filled densities show generated molecules for each model (HGR variants, green; baselines, blue), and black contours show the test-set density.
}
\label{fig:gen-rel}
\end{figure}

To test whether generated molecules retain chemically meaningful substructures, we profiled each molecule against a predefined functional-group panel and recorded the fraction of molecules containing each group (SMARTS definitions and procedure in \cref{SI-sec:fg}).
This local-composition analysis shows that both HGR generation routes preserve the broad functional-group profile of the RingDiv300k reference test set (\cref{fig:gen-rel}\textbf{a,b}).
Nitrogen-containing groups and aromatic motifs showed the closest agreement, with the dominant pyridine essentially reproduced and most groups agreeing to within about one percentage point. 
HGR-LDF further reduced the ether enrichment observed under HGR-VAE. 
Residual deviations were confined to a few motifs: benzene and carbonyl-containing groups such as amides and ketones were enriched, whereas thioethers, cyclopropanes, epoxides and rare aldehydes were under-represented.
Overall, grammar-constrained generation preserves the major functional-group composition of the reference set, although a few motif-level deviations remain.

We next assessed whether this fidelity persisted across molecular size by binning molecules by heavy-atom count and measuring per-bin agreement with the test set (\cref{fig:gen-rel}\textbf{c--e}).
For the ring-size distribution (\cref{fig:gen-rel}\textbf{c}), both HGR variants stayed close to the train-versus-test floor at all sizes (JS $\leq 0.01$).
By contrast, the graph-diffusion baselines showed larger deviations on large molecules, with DiGress rising about tenfold to JS $\approx 0.09$ above 45 heavy atoms.
The ranking differed for functional-group composition (\cref{fig:gen-rel}\textbf{d}): the ring-specialist HoG-Diff and DeFoG were less accurate (JS $\approx 0.02-0.03$), whereas HGR remained low (JS $\approx 0.01-0.015$) and comparable to the most accurate baseline, DiGress.
HGR showed the most consistent agreement with the reference across both ring topology and local composition.
Because ring-rich molecules can be synthetically challenging, we next examined size-resolved synthetic accessibility~\cite{SAscore-2009} (\cref{fig:gen-rel}\textbf{e}). 
HGR-generated molecules retained low SA scores across the size range, whereas diffusion baselines became less accessible with size, with DiGress reaching SA $\approx 6$.
Together, the size-resolved analysis shows that grammar-constrained generation preserves chemical fidelity and synthetic feasibility on large, ring-rich molecules for which graph-diffusion models showed larger deviations.

We further compared chemical-space coverage using two-dimensional Uniform Manifold Approximation and Projection (UMAP)~\cite{umap-NRMP2024} density maps of Morgan fingerprints~\cite{morgan_fp2010} (\cref{fig:gen-rel}\textbf{f}).
Because Morgan fingerprints encode local substructures, these density maps provide a compact and intuitive comparison of reference and generated chemical-space distributions.
The HGR variants showed broad overlap with the test-set density, whereas the diffusion baselines showed more restricted or shifted density patterns.

\subsection*{Molecular foundation model with HGR}
\pdfbookmark[2]{Molecular foundation model with HGR}{sec3.5}

\begin{figure}
\centering
\includegraphics[width=0.99\linewidth]{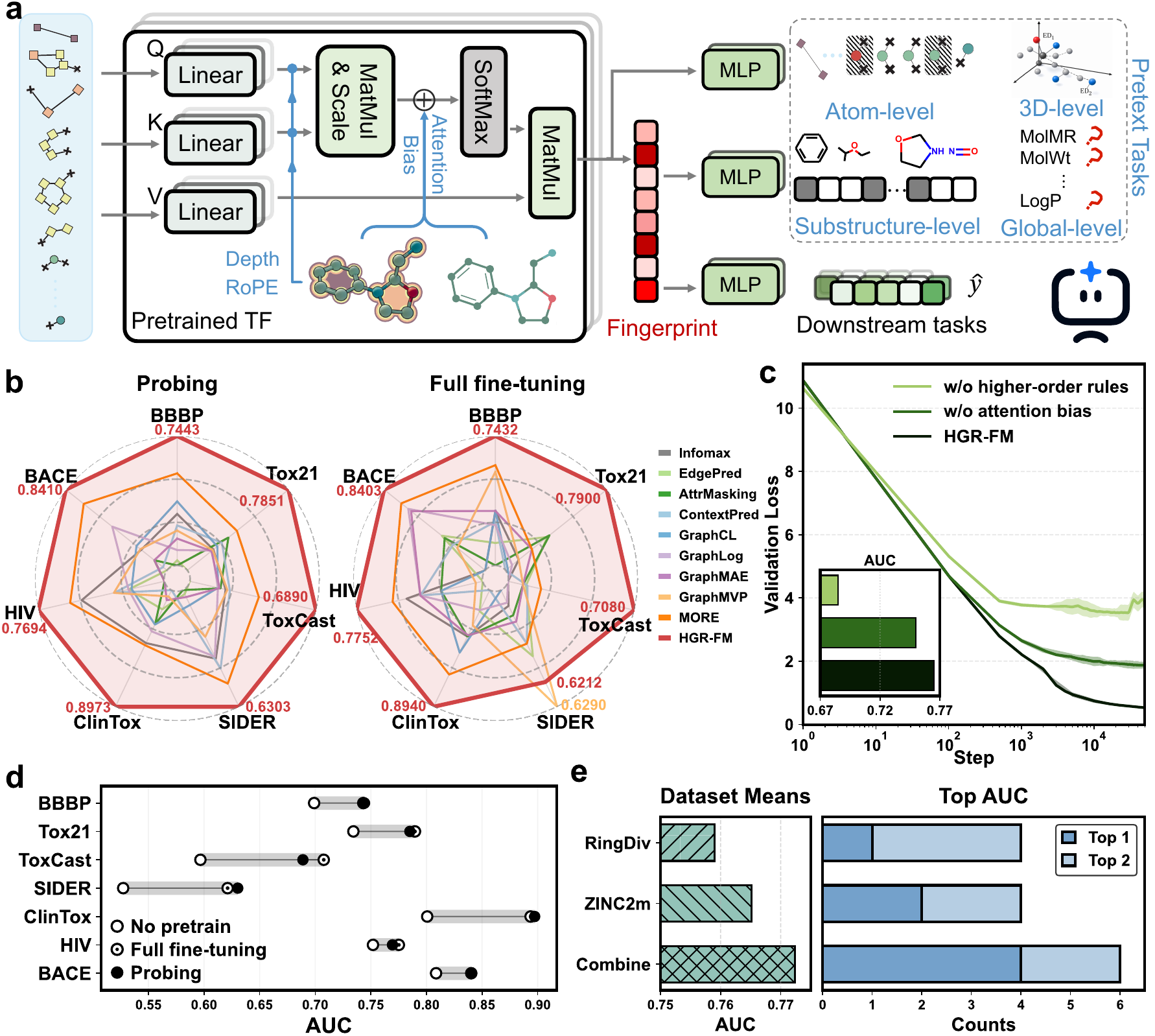}
\caption{{\bf Workflow and transfer performance of the HGR-based molecular foundation model.}
{\bf a}, Overview of HGR-based pretraining and downstream transfer.
A pretrained Transformer with rank-aware rotary position encoding (RoPE) and an HGR attention bias yields a molecular fingerprint, jointly optimised with atom-, substructure-, 3D- and global-level pretext tasks and transferred to downstream tasks via lightweight multilayer-perceptron (MLP) task heads.
{\bf b}, Normalised AUCs after ZINC2m pretraining under probing (left) and full fine-tuning (right) across seven MoleculeNet benchmarks. AUCs are min--max normalised per task across all methods; annotated values are the raw HGR-FM AUCs.
{\bf c}, Ablation of higher-order rule tokens and attention bias. Curves show validation loss during pretraining for the full HGR-FM model and variants without each component. Inset, mean probing AUC across seven downstream tasks.
{\bf d}, Comparison of training regimes: no pretraining, probing and full fine-tuning. %
{\bf e},  Effect of pretraining data source, comparing RingDiv, ZINC2m and their combination. Left, mean probing AUC; right, number of downstream tasks for which each setting ranks first or second.
}
\label{fig:FM-rel}
\end{figure}

Beyond generation, the same HGR string can serve as the substrate for a foundation model.
Molecular foundation models (FMs) pretrained on large corpora can learn transferable representations for downstream tasks with limited task-specific supervision~\cite{MolFormer-NMI2022,SPMM-NC2024,MolE-NC2024}.
However, current molecular FMs largely depend on SMILES strings\cite{MolFormer-NMI2022,SPMM-NC2024}, molecular graphs\cite{KGPT-NC2023,MolE-NC2024,MORE-AAAI2025}, or simple combinations of the two\cite{DVMP-KDD2023}.
These SMILES- and graph-centred inputs mainly encode atomic connectivity, leaving hierarchical higher-order patterns underrepresented.
Simply scaling models on such inputs is therefore unlikely to overcome this representational bottleneck, limiting their potential in chemistry and drug discovery~\cite{rottach2025topology}.

To address these gaps, we propose an HGR-based molecular foundation model (HGR-FM) (\cref{fig:FM-rel}\textbf{a}) that uses compact HGR strings as structurally expressive inputs and is pretrained for multi-task transfer.
Each rule token is initialised with a structure-derived embedding from Weisfeiler-Lehman (WL) label histograms.
This embedding is permutation-invariant and generalises to unseen rules.
A Grammar Transformer then encodes these token embeddings into contextualized representations.
Because vanilla Transformers do not exploit hierarchical or topological structure by default, we introduce a higher-order position encoding implemented as a composite attention bias.
This bias injects cues for cell rank, grammar-tree distance and relation type, together with pairwise molecular connectivity, into the pre-softmax attention logits.
Contextualized atom-level token states support atom-level prediction, whereas the final hidden state of the molecule-level token serves as a fixed-length molecular fingerprint.
Finally, multi-level pretext tasks align atom-, fragment-, molecule- and conformation-level signals, with shallow task-specific multilayer perceptron (MLP) heads projecting the corresponding representations to their training targets.
Model architecture, rule-token initialisation and pretraining objectives are detailed in \cref{SI-sec:HGR-FM}.

After pretraining, we evaluated HGR-FM across various unseen downstream tasks using two standard transfer protocols: full fine-tuning and probing.
Performance was reported as AUC (area under the receiver operating characteristic curve), following the MoleculeNet convention.
Across seven MoleculeNet benchmarks (BBBP, Tox21, ToxCast, SIDER, ClinTox, HIV and BACE)~\cite{moleculenet-wu2018}, HGR-FM achieved the highest average AUC under both transfer protocols (\cref{fig:FM-rel}\textbf{b}). At the task level, HGR-FM led all seven benchmarks under probing and six under full fine-tuning, with SIDER as the only exception. 
The largest probing gain was observed on ClinTox (89.7 versus 63.5 for the strongest external baseline).
Complete task-level probing and full fine-tuning results are provided in~\cref{SI-tab:FM_transfer_combined}.
Comparing the two transfer protocols, probing approached or surpassed full fine-tuning (\cref{fig:FM-rel}\textbf{d}).
This pattern may reflect full fine-tuning's instability on small downstream datasets: updating the pretrained encoder can induce representation drift, whereas freezing it acts as an implicit regularizer.
The narrow gap between probing and fine-tuning indicates that the pretrained encoder learns semantically rich and transferable representations from which shallow prediction heads can read out task-relevant features.

To isolate the source of these gains, we ablated the two HGR-specific components, first removing higher-order tokens from the HGR string to leave only atom-level tokens.
Experimental results show that higher-order rule tokens are important for the observed gains: removing them degraded pretraining dynamics, yielding consistently higher validation loss and a clear drop in downstream mean AUC (\cref{fig:FM-rel}\textbf{c}).
Disabling the attention bias also degraded performance, but the effect was smaller than that of removing higher-order rule tokens.
These results support the core design premise: tokenizing molecules with higher-order rules exposes meso-scale and global structural cues to the Transformer, letting it learn representations that go beyond atom-level co-occurrence and better reflect hierarchical molecular topology.

We next assessed the impact of pretraining data by comparing models trained on RingDiv, ZINC2m and their combination (\cref{fig:FM-rel}\textbf{e}).
ZINC2m mainly comprises simpler molecular structures that capture fundamental chemical patterns and suffice for many benchmark tasks, whereas RingDiv contributes more complex and topologically rich molecules that broaden the model’s exposure to higher-order structures.
Combining ZINC2m with RingDiv yielded the best downstream performance, highlighting the value of larger and structurally broader pretraining data.

Together, these results show that HGR-FM integrates symbolic production rules with explicit structural inductive biases to produce transferable and chemically informed representations.
Its strong probing performance shows that grammar-informed pretraining learns task-relevant molecular features efficiently, positioning HGR-FM as a promising foundation model framework for molecular representation learning.

\section*{Discussion}
\pdfbookmark[1]{Discussion}{sec:Discussion}

This work bridges topological deep learning and sequence modelling by encoding hierarchical molecular topology as a decodable string that standard sequence models can learn from.
Building on a combinatorial-complex view of molecules, we examined three linked aspects: reusable grammar induction, ring-diverse benchmarking, and HGR-based generation and foundation modelling.
These analyses show that explicit higher-order topology can support chemically valid generation, improve distributional alignment on ring-diverse chemical space, and yield transferable molecular representations.
Ablations that remove the higher-order rule tokens, leaving only atom-level tokens, degrade transfer (\cref{fig:FM-rel}\textbf{c}), indicating that the gains stem from higher-order topology rather than from atom-level co-occurrence alone.

Conventional molecular benchmarks underrepresent complex ring systems, whereas modern drug discovery increasingly involves fused, bridged and multi-ring architectures.
This imbalance can hamper generalisation to ring-diverse chemical space.
Performance saturation on conventional benchmarks can also mask methodological differences that surface only under more discriminative evaluation.
RingDiv and HGR address the two sides of this gap: RingDiv enriches evaluation with 1.18 million ring-diverse molecules, and HGR converts molecular combinatorial complexes into production-rule sequences that can be modelled by standard sequence architectures.
In particular, HGR contrasts with incidence-matrix encodings of higher-order topology, which represent structure explicitly but can make molecular generation computationally costly.
Together, RingDiv and HGR pair a more demanding evaluation regime with a representation that exposes reusable higher-order motifs for generation and transfer learning.

The present study also identifies several practical limitations of HGR and its current evaluation.
Computationally, the lifting step is the main upfront cost for very large corpora, although it runs once, parallelises well and is amortised across subsequent training and generation.
This cost is further mitigated downstream, as compact HGR strings enable HGR-VAE to train and sample faster than the diffusion baselines.
Generation remains sensitive to local chemical composition, with motif-level biases such as ether enrichment under direct latent sampling. Latent-space diffusion reduced this enrichment, and the remaining distributional gaps identify motifs for which grammar-constrained sampling could be improved.
Finally, the generated molecules and predicted candidates remain computational proposals until prospective wet-lab assays establish their practical value.
Beyond addressing these boundaries, future work can extend HGR to other chemical tasks.
For instance, localised edits to HGR rules could support molecular editing, while substituting rules that encode core ring systems could enable scaffold hopping.

Beyond these specific tasks, HGR offers a reusable structural vocabulary that makes the higher-order topology of molecules explicit and learnable by encoding how local chemical units assemble into meso-scale motifs and whole-molecule architectures. 
RingDiv provides a benchmark for evaluating molecular models on complex, ring-rich chemistry.
More broadly, this work suggests a route towards applying LLM-style sequence models to graph- and topology-centric scientific domains by serialising higher-order structures into production-rule sequences.

\section*{Methods}
\pdfbookmark[1]{Methods}{sec:Methods}

\subsection*{Preliminaries and problem formulation}
\pdfbookmark[2]{Preliminaries and problem formulation}{sec:5.prelim}

We introduce the HGR string as a compact, invertible sequence representation of higher-order molecular topology.
Starting from a molecular graph $\mathcal{G}$, we first lift it to a combinatorial complex $\mathcal{C}$ that encodes hierarchical higher-order structure.
Grammar induction then uses both $\mathcal{C}$ and $\mathcal{G}$ to produce an HGR string $\bm{r}$.

A molecule can be represented as a graph $\mathcal{G}=(\mathcal{V},\mathcal{E})$, with atoms as vertices $\mathcal{V}$ and covalent bonds as edges $\mathcal{E}$. However, molecules often contain multi-way relations such as rings, functional groups and other higher-order motifs that pairwise edges cannot fully capture~\cite{HigherOrderReview2020,RingSystems-EJMC2025,erlanson2016twenty}. We therefore represent a molecule as a \emph{combinatorial complex} (CC)~\cite{CCs-TDL-2022}, a general formalism that unifies graphs, simplicial and cell complexes, and hypergraphs within a single framework. A CC is a triple
\begin{equation}
    \mathcal{C}=(\mathcal{V},\mathcal{X},\rk),
\end{equation}
where $\mathcal{V}$ is a finite set of vertices, $\mathcal{X} \subseteq \operatorname{Pow}(\mathcal{V}) \setminus \{\emptyset\}$ is a collection of non-empty subsets called \emph{cells} (or relations), and $\rk: \mathcal{X} \to \mathbb{Z}_{\geq 0}$ is a \emph{rank} function. The structure satisfies two conditions: every singleton $\{v\}$ is a cell with $\rk(\{v\}) = 0$, so atoms have rank $0$; and $\rk$ is order-preserving, i.e.\ $x \subseteq y \Rightarrow \rk(x) \leq \rk(y)$.
In HGR, the lifted CC encodes the higher-order hierarchy, while the molecular graph $\mathcal{G}$ supplies the pairwise adjacency used during rule construction; the two roles stay separate throughout.

\subsection*{Molecular lifting algorithms}
\pdfbookmark[2]{Molecular lifting algorithms}{sec:5.2}

We now specify how molecular graphs are lifted into combinatorial complexes.
Let $\mathcal{G}=(\mathcal{V},\mathcal{E})$ be a molecular graph and let $\mathcal{M}$ denote the motif vocabulary mined from a molecule library. 
For each lifting variant $\eta\in\{\mathrm{MIG},\mathrm{RSG}\}$, molecular lifting defines a fixed, data-driven structural map
\begin{equation}
    \Lambda_{\eta}: (\mathcal{G},\mathcal{M})
    \longmapsto
    \mathcal{C}_{\eta}=(\mathcal{V},\mathcal{X}_{\eta},\rk_{\eta}),
\end{equation}
where $\mathcal{X}_{\eta}$ is the cell family recorded by the lifting hierarchy, containing singleton atom cells and the lifted motif cells produced by merging.
The molecular graph $\mathcal{G}$ is retained alongside $\mathcal{C}_{\eta}$ and later supplies the pairwise adjacency used to define the anchor interfaces and local rule graphs.

Both lifting variants are built on a frequency-guided motif-merging backbone.
Starting from singleton atom cells, the algorithm scores adjacent candidate subgraphs by their frequency and merges the highest-frequency non-overlapping candidates in parallel, yielding progressively coarser motif cells.
The merges form a contraction tree whose leaves are atom cells.
For any cell $\sigma$, its rank $\rk_{\eta}(\sigma)$ is defined as its height in this tree, namely the maximum number of edges from $\sigma$ to an atom leaf.
This definition assigns rank $0$ to atom cells and makes rank non-decreasing under cell inclusion: $\sigma\subseteq\tau$ implies $\rk_{\eta}(\sigma)\leq\rk_{\eta}(\tau)$.

A key practical challenge is controlling the number of boundary anchors and the resulting attachment complexity.
For a merged cell $\sigma$, its anchors are the external neighbours of $\sigma$ in the molecular graph $\mathcal{G}$. More anchors give rise to more attachment configurations and can inflate the induced rule corpus.
The two lifting variants specialise this shared backbone in different ways.
MIG adds anchor resolution to simplify complex interfaces, whereas RSG stages merging and anchor resolution around ring systems to preserve them as coherent higher-order units during lifting.

\paragraph{MIG lifting.} MIG is designed to reduce excessive anchors. When a merged motif exposes an ambiguous interface, defined here as more than two anchors, MIG expands the motif by absorbing neighbouring subgraphs in a best-first manner, prioritising higher-frequency motifs. This absorption continues until the interface becomes unambiguous. MIG typically decreases the number of induced rules, at the cost of producing slightly larger motifs. The merging procedure and its anchor-resolution subroutine are summarised in \cref{SI-alg:mig_lifting}.

\paragraph{RSG lifting.}
RSG further stabilises lifting in cyclic chemistry by prioritising ring-system-aware merges.
It first identifies rings using the smallest set of smallest rings decomposition and groups rings that share atoms into common ring systems.
RSG then applies constrained merging in stages. Non-ring regions are first merged to consolidate acyclic fragments; side chains that attach uniquely to a ring system are absorbed into that system; within-system merges are prioritised before cross-system merging is allowed to preserve ring coherence; and each ring system is treated as a unit for anchor resolution, allowing necessary neighbouring vertices to be absorbed while keeping the interface manageable. A final global merging pass is then applied.
This staged design yields a more consistent hierarchy for ring-rich molecules and mitigates rule inflation caused by complex ring attachments.

\subsection*{Higher-order grammar induction and derivation}
\pdfbookmark[2]{Higher-order grammar induction and derivation}{sec:5.grammar}

Intuitively, a \emph{terminal symbol} denotes an observed atom-level object, whereas a \emph{non-terminal symbol} marks a replaceable site corresponding to a contracted cell.
A \emph{production rule} rewrites a non-terminal site into a local structure when its left-hand-side anchor interface matches the site's interface.
Formally, the grammar is $\Gamma = (\Sigma, \mathcal{N},  {S}, \mathcal{P})$, where $\Sigma$ and $\mathcal{N}$ denote the finite sets of terminal and non-terminal symbols, respectively, ${S}\in\mathcal{N}$ is the start symbol, and $\mathcal{P} = \{r_1, r_2, \dots, r_{|\mathcal{P}|}\}$ is the corpus of production rules.
For each complex produced by either lifting variant, \emph{parsing} returns an ordered HGR string $\bm{r}$ of rule tokens drawn from $\mathcal{P}$.
Across the library, \emph{induction} accumulates the rules created during parsing into the production-rule corpus $\mathcal{P}$.
The reverse direction, \emph{derivation}, rebuilds a complex from the start symbol $S$.
The overall induction-derivation workflow is illustrated in \cref{SI-fig:grammar_induction}.

\paragraph{Production-rule construction.}
We consider parsing a single lifted complex $\mathcal{C}$ paired with its molecular graph $\mathcal{G}$.
Rank-zero cells correspond to atoms, whereas higher-rank cells are the internal nodes of the contraction tree.
We serialise this hierarchy in a bottom-up post-order, completing each child subtree before its parent cell, with admissible sibling-subtree orders defined in Supplementary Methods.
Following the selected order, we contract each cell $\sigma\in\mathcal{X}$ into a fresh site marked by the non-terminal symbol $N^*$, producing one rule token $r_i$ for the HGR string.
After $r_i$ is recorded, $N^{*}$ represents $\sigma$ in the current complex and graph, reducing the active structure while allowing this site to participate in subsequent contractions.
Repeating this contraction step reduces the structure to a single root site, which is marked by the start symbol ${S}$.

Each production rule $r_i:\mathcal{G}^L_i\to\mathcal{G}^R_i$ comprises a left-hand-side interface graph and a right-hand-side local graph (\cref{fig:HGR_frame}\textbf{d}). The ordered anchor interface specifies both whether the rule is applicable and how its local rewrite attaches to the surrounding graph. 
We define the anchor set as
\begin{equation}
    \mathcal{V}_{\mathrm{anc}}^\sigma
    =
    \{v\in\mathcal{V}\setminus\sigma
    \mid
    \exists u\in\sigma\text{ s.t. }(u,v)\in\mathcal{E}\}.
\end{equation}
Its canonical ordering and corresponding boundary-bond attributes define the ordered anchor interface. When $\sigma$ is contracted, this interface is retained at the site marked by $N^*$, thereby fixing the attachment correspondence for reverse derivation.
Accordingly, the left-hand side is the star-shaped interface graph
\begin{equation}
    \mathcal{G}^L_i
    =
    \left(
    \{N^*\}\cup\mathcal{V}_{\mathrm{anc}}^\sigma,
    \{(N^*,v)\mid v\in\mathcal{V}_{\mathrm{anc}}^\sigma\}
    \right),
\end{equation}
where the anchor ordering and boundary-bond attributes are retained on the incident edges.

Let $\widetilde{\mathcal{G}}=(\widetilde{\mathcal{V}},\widetilde{\mathcal{E}})$ denote the working graph at the current contraction step and $\widetilde{\sigma}\subseteq\widetilde{\mathcal{V}}$ the current representation of $\sigma$.
For $U\subseteq\widetilde{\mathcal{V}}$, we write $\widetilde{\mathcal{G}}[U]=(U,\{(u,v)\in\widetilde{\mathcal{E}}\mid u,v\in U\})$ for the subgraph induced by $U$, with all vertex and edge attributes inherited from $\widetilde{\mathcal{G}}$. The corresponding production rule has right-hand side
\begin{equation}
    \mathcal{G}^R_i
    =
    \widetilde{\mathcal{G}}
    \bigl[\widetilde{\sigma}\cup\mathcal{V}_{\mathrm{anc}}^\sigma\bigr].
\end{equation}
Thus, $\mathcal{G}^R_i$ records the local structure replaced by $N^*$: the atoms and child non-terminal sites currently representing $\sigma$, together with the boundary bonds connecting this representation to the ordered anchors.
Structurally identical local contractions reuse the same production rule in $\mathcal{P}$. 
Together, the rule order in the HGR string and the ordered child non-terminal sites in each $\mathcal{G}^R_i$ specify the derivation tree used for reverse reconstruction.

\paragraph{Reverse derivation and lossless reconstruction.}
Given the induced grammar $\Gamma$ and an HGR string $\bm{r}$, reverse derivation reconstructs the combinatorial complex and associated molecular graph (\cref{fig:HGR_frame}\textbf{f}).
Because the HGR string is stored in post-order, reverse derivation starts from ${S}$ and consumes $\operatorname{rev}(\bm{r})$, beginning with the root rule and proceeding towards the leaves.
At each step, the current rule $r_i$ expands the non-terminal site at the top of the decoder's last-in-first-out (LIFO) stack, which selects the rewrite site without global subgraph-isomorphism search.
The rule rewrites $\mathcal{G}^L_i$ into $\mathcal{G}^R_i$ and reconnects boundary bonds through the ordered anchor interface, which fixes the boundary correspondence even for symmetric intermediates.
Thus, reverse derivation losslessly restores both the input lifted complex $\mathcal{C}_{m}$ and its associated molecular graph $\mathcal{G}_{m}$ (Supplementary Proposition~\ref{SI-prop:lossless}).

\paragraph{Rule-token initialisation.}
For use with neural models, each production-rule token in the HGR string is mapped to an input vector.
Rather than one-hot rule identifiers, which treat rules as unrelated symbols, we initialise each rule token $r_i$ from a structure-aware Weisfeiler--Lehman (WL) histogram of its right-hand-side graph $\mathcal{G}^R_i$.
The WL histogram iteratively relabels each atom by combining its attributes with the sorted labels of its neighbours.
Pooling the resulting labels yields a fixed-length, permutation-invariant feature that generalises to rules unseen during training.
These WL features already give distinct initialisations to more than $98\%$ of production rules.
In HGR-FM, these features are augmented with a random-walk encoding, so the combined embedding distinguishes all rules in our grammars.
Full construction and model-specific initialisers are detailed in \cref{SI-sec:rule-init}.

\subsection*{HGR-based generative models}
\pdfbookmark[2]{HGR-based generative models}{sec:5.3}

With the induced grammar and its rule-token initialisation in hand, we build a grammar VAE over HGR strings, together with an optional latent-diffusion sampler over its latent space.

\paragraph{Grammar variational autoencoder (GVAE).}
The GVAE consists of an encoder--decoder pair that operates on grammar production rules rather than atomic tokens~\cite{GVAE+ICML2017}.
Given a molecular combinatorial complex $\mathcal{C}$, we first parse it into an HGR string $\bm{r}$ whose rule tokens are drawn from the production rule corpus $\mathcal{P}$.
Each rule token is then mapped to its WL-based embedding, so that the symbolic string $\bm{r}$ becomes a sequence of fixed-length feature vectors.
From this feature sequence, the encoder produces a latent variable $\bm{z} \in \mathbb{R}^d$, parameterising the variational posterior $q_{\bm{\phi}}(\bm{z} \mid \bm{r})$.
The latent prior is an isotropic Gaussian, $p(\bm{z}) = \mathcal{N}(\bm{0}, \bm{I})$.
Conditioned on $\bm{z}$, the decoder represents $p_{\bm{\theta}}(\bm{r}\mid\bm{z})$ using an autoregressive factorization over $\operatorname{rev}(\bm{r})$, beginning with the root rule and proceeding towards the leaves under grammar constraints.

The model is trained with a $\beta$-VAE style evidence lower bound over the dataset of HGR strings~\cite{Higgins2017betaVAE}:
\begin{equation}
\mathcal{L}_{\mathrm{GVAE}}(\bm{\theta}, \bm{\phi}; \bm{r}) =
\mathbb{E}_{q_{\bm{\phi}}(\bm{z}\mid\bm{r})}[\log p_{\bm{\theta}}(\bm{r}\mid\bm{z})]
-
\beta\,\mathrm{KL}\!\left(q_{\bm{\phi}}(\bm{z}\mid\bm{r})\,\|\,p(\bm{z})\right).
\end{equation}
The reconstruction term encourages recovery of the input HGR string, whereas the Kullback--Leibler (KL) divergence regularises the latent space; $\beta$ controls the trade-off between reconstruction fidelity and latent regularisation. Variable-length sequences are padded to a common length with a special padding token for mini-batch training.

\paragraph{Grammar-constrained decoding.}
To ensure syntactic validity, decoding is guided by a last-in-first-out (LIFO) stack of pending non-terminal sites.
At step $t$, the stack top selects the site $\alpha_t$ to expand.
A production rule is admissible only if its ordered anchor interface matches the interface at $\alpha_t$.
The decoder outputs an unnormalized logit vector $\bm{l}_t\in\mathbb{R}^{|\mathcal{P}|}$, and a binary mask $\bm{m}_{\alpha_t}\in\{0,1\}^{|\mathcal{P}|}$ encodes rule admissibility.
The probability of selecting rule $r_i$ given the preceding rules $\bm{r}_{<t}$, $\alpha_t$ and the latent variable $\bm{z}$ is
\begin{equation}
p(r_t=r_i\mid\bm{r}_{<t},\alpha_t,\bm{z})
=
\frac{\bm{m}_{\alpha_t,i}\exp(\bm{l}_{t,i})}
{\sum_j\bm{m}_{\alpha_t,j}\exp(\bm{l}_{t,j})}.
\end{equation}
After a rule is selected, its child non-terminal sites are pushed onto the stack in reverse of their stored order.
Decoding continues until the stack is empty or the dataset-specific maximum sequence length is reached.
Sequences with non-empty stacks at this limit are discarded as incomplete derivations.
Each completed derivation is therefore grammar-valid and satisfies the local valency and bonding constraints encoded by the induced rules, without post-hoc filtering for these constraints.

\paragraph{Latent diffusion sampling.}
To further improve sample quality, we model the GVAE latent distribution with a standard score-based diffusion model~\cite{NCSN+NeurIPS2019, Score-SDE+ICLR2021}. Latent codes are perturbed by a forward stochastic differential equation (SDE) and recovered by the corresponding reverse-time SDE, whose score function is trained by denoising score matching~\cite{Scorematching2011}.
We adopt this framework without modification and provide the full formulation in \cref{SI-sec:latent-diffusion}.

Consequently, latent codes are sampled either directly from the prior $p(\bm{z})$ or through the diffusion process.
In both cases, greedy autoregressive decoding maps each latent code to a grammar-admissible HGR rule sequence, after which deterministic reverse derivation reconstructs the corresponding combinatorial complex and associated molecular graph.
The resulting molecules preserve the hierarchical and higher-order topology encoded by HGR strings.

\subsection*{HGR-based foundation model (HGR-FM)}
\pdfbookmark[2]{HGR-based foundation model (HGR-FM)}{sec:5.4}

HGR-FM (\cref{fig:FM-rel}\textbf{a}) takes the HGR string $\bm{r}$ of a molecule as input and maps its rule tokens to initialised embeddings.
The resulting token embeddings are encoded by a Grammar Transformer whose self-attention uses the higher-order position bias defined below.
Contextualised atom-level token states feed atom-level pretext tasks, whereas the final hidden state of the molecule-level token serves as a fixed-length molecular fingerprint for higher-level tasks and downstream transfer.

\paragraph{Higher-order position encoding.}
Transformers do not explicitly exploit input structure by default~\cite{transformer-NeurIPS2017}, whereas appropriate structural inductive biases can improve their performance~\cite{Graphormer-NeurIPS2021,GRIT-ICML2023}.
To capture the hierarchical geometry together with pairwise connectivity, we introduce a composite structural encoding into self-attention.
For each attention head, the pre-softmax attention logit between tokens $i$ and $j$ takes the form
\begin{equation}
\text{Logit}_{i,j} =
\underbrace{
    \bigl(\mathcal{R}(\rk(\sigma_i))\bm{q}_i\bigr)^\top
    \bigl(\mathcal{R}(\rk(\sigma_j))\bm{k}_j\bigr)/\sqrt{d_h}
    + b_{\mathcal{C}}(\delta^{\mathcal{C}}_{ij})
    + b_{\rho}(\rho^{\mathcal{C}}_{ij})
}_{\text{Higher-order structural encoding}}
+
\underbrace{
    \mathbb{I}_{i,j\in\mathcal{G}}
    \left(
        \phi(\delta^{\mathcal{G}}_{ij}) + \psi_{ij}
    \right)
}_{\text{Pairwise structural encoding}}.
\end{equation}
Here, $d_h$ is the per-head dimension, $\sigma_i$ denotes the cell associated with token $i$, and $\mathcal{R}(\cdot)$ denotes Rotary Positional Embedding (RoPE)~\cite{ROPE-RoFormer-Su2024} applied to the query $\bm{q}_i$ and key $\bm{k}_j$ using their cell ranks.
The resulting inner product encodes their relative rank difference.
The learnable scalar biases $b_{\mathcal{C}}$ and $b_{\rho}$ are shared across all layers and are indexed by the grammar-tree distance $\delta^{\mathcal{C}}_{ij}$ and relation type $\rho^{\mathcal{C}}_{ij}$, respectively.
The relation types comprise self, parent, child, sibling, ancestor, descendant and other.
The pair $(\rk(\sigma_i), \delta^{\mathcal{C}}_{ij})$ forms a lossless signature of hierarchical structure for molecular complexes $\mathcal{C}$, enabling precise discrimination of topological relations such as siblings and cousins.

We also add pairwise structural biases using shortest-path spatial encoding and edge-path encoding~\cite{Graphormer-NeurIPS2021}.
Pairwise spatial encoding is implemented as a learnable bias $\phi(\delta^{\mathcal{G}}_{ij})$, indexed by the shortest-path distance $\delta^{\mathcal{G}}_{ij}$ on the molecular graph $\mathcal{G}$.
Edge-path encoding $\psi_{ij}$ is defined as the average of edge features aggregated along the shortest path $\text{SP}_{ij}$
\begin{equation}
    \psi_{ij}
= \frac{1}{|\text{SP}_{ij}|}\sum_{e\in \mathrm{SP}_{ij}} \bm{x}_{e}^{\top}\bm{w}^{\mathcal{E}}_{\text{pos}_{ij}(e)},
\end{equation}
where $\text{SP}_{ij}$ denotes edge sequence on the shortest path from $i$ to $j$,  $\text{pos}_{ij}(e)$
indicates the relative position of edge $e$ along this path, $\bm{x}_e$ is the feature of edge $e$, and $\bm{w}^{\mathcal{E}}_{n}$ is a learnable embedding for path position $n$.
Pairwise biases are applied only to atom-token pairs, and are set to zero otherwise.
Attention heads are partitioned across the additive encodings, so each head receives the content score together with a subset of these terms (\cref{SI-sec:HGR-FM}).
Together, these terms inject hierarchical cues and pairwise connectivity into attention while ensuring perturbation invariance.

\paragraph{Multi-level pretext tasks.}
HGR-FM is pretrained with four complementary pretext tasks spanning atom-, fragment-, molecular- and 3D-conformation-level supervision~\cite{MORE-AAAI2025}.
The atom-level task masks atom-level rule tokens and predicts their atomic attributes from contextualised token representations, encouraging fine-grained local chemical perception.
The fragment-level task predicts 1{,}024-bit Morgan~\cite{morgan_fp2010} (radius 2) and 166-bit Molecular ACCess System (MACCS)~\cite{maccs-durant2002} binary fingerprints that represent the presence or absence of specific functional groups or substructures, thereby encoding chemically meaningful local motifs.
The molecular-level task regresses molecular descriptors (e.g., molecular weight, logP), encouraging the model to internalize global physicochemical properties.
Finally, the 3D-level task predicts pairwise interatomic distances derived from optimised molecular conformers, enabling integration of spatial geometry.
These tasks are jointly optimised using the same task weights as MORE~\cite{MORE-AAAI2025}.
Together, the four tasks yield transferable representations spanning local, substructural, global-property and geometric information (see~\cref{SI-sec:pretext_tasks} for details).

\subsection*{Datasets and benchmarks}
\pdfbookmark[2]{Datasets and benchmarks}{sec:5.data}

\paragraph{RingDiv benchmark construction.}
We constructed RingDiv, a ring-enriched, topology-balanced benchmark comprising 1.18 million molecules, together with a curated subset, RingDiv300k, for rapid development and ablation studies. The ring-topology coverage of both datasets is characterised in \cref{fig:ringdiv-analysis}.
It was designed to broaden coverage of complex ring architectures while retaining balanced distributions of ring size and fusion topology in drug-like chemical space.
The source pool contained approximately 143 million compounds from PubChem~\cite{kim2016pubchem}, SureChEMBL~\cite{papadatos2016surechembl}, and five commercial screening libraries (eMolecules, Asinex, ChemDiv, Enamine and MolPort).
After merging, molecules were standardised as RDKit molecular objects; unparsable entries were discarded, and repeated source records were collapsed using InChI identifiers~\cite{heller2015inchi} generated from the standardised structures.
For each molecule, we performed ring perception and annotated ring-size and fusion-topology descriptors: ring sizes were binned as small (3--4 atoms), normal (5--7 atoms), medium (8--11 atoms), or macro ($\geq$ 12 atoms), and ring systems were categorised by connectivity pattern (spirocyclic, fused, condensed, or bridged).

We then constrained the chemical space to a drug-like regime by retaining compounds with 1--10 rings and molecular weights of 100--900 Da, and excluding molecules containing macrocycles.
Macrocycles are sparsely represented in the source databases and differ markedly from smaller ring systems in flexibility, synthetic accessibility, and physicochemical properties; their inclusion would introduce poorly supported outliers.
These per-molecule descriptors were used to guide topology-balanced sampling and to summarise ring-topology coverage in the final benchmark.

To ensure adequate representation of structurally informative but underrepresented motifs, we adopted a two-stage sampling strategy that combines topology balancing with fingerprint-based diversity selection.
First, a quota-first pass prioritised rare classes such as medium rings, spirocyclic and bridged architectures: molecules were sampled to satisfy minimum quotas for these classes, after which the remaining pool was randomly expanded into an oversampled candidate set while preserving the desired marginal topology balance.
Second, to maximise structural diversity and reduce near-duplicates, candidates were encoded using 2{,}048-bit Morgan~\cite{morgan_fp2010} fingerprints (radius 2), and a MaxMin diversity criterion was applied using Tanimoto similarity.
This step favours subsets that are globally well dispersed in fingerprint space while discouraging highly similar neighbours, using an approximate Tanimoto-similarity threshold of 0.65 without compromising the intended ring-topology coverage.
RingDiv300k was obtained from RingDiv by topology-balanced subset selection, using tighter quotas for rare ring categories rather than random subsampling.

\paragraph{Datasets for generation and transfer.}
We evaluated de novo molecular generation on five benchmarks: QM9~\cite{data:qm9}, ZINC250k~\cite{data:zinc250k}, MOSES~\cite{moses-2020}, GuacaMol~\cite{guacamol-JCIM2019} and RingDiv300k. The final two benchmarks were used for the main large-scale comparisons.
For foundation-model pretraining, we used three corpora: ZINC2m~\cite{data:zinc250k}, RingDiv and their combination (ZINC2m+RingDiv). ZINC2m was the primary setting used for baseline alignment, whereas RingDiv and the combined corpus were used to assess the effect of pretraining data.
Downstream transfer was evaluated on seven MoleculeNet classification benchmarks~\cite{moleculenet-wu2018}: BBBP, Tox21, ToxCast, SIDER, ClinTox, HIV and BACE, each split by scaffold into 80/10/10 training, validation and test partitions~\cite{MORE-AAAI2025}.

\subsection*{Experimental setup}
\pdfbookmark[2]{Experimental setup}{sec:5.expsetup}

\paragraph{Generation baselines and metrics.}
Generation baselines span sequence-, motif-, graph- and topology-aware molecular generators. 
For the main-text comparison on GuacaMol and RingDiv300k, we use the recent diffusion and higher-order-aware baselines DiGress~\cite{DiGress+ICLR2023}, HoG-Diff~\cite{hogdiff-ICLR2026}, DeFoG~\cite{DeFoG-ICML2025} and SELFIES-VAE~\cite{SELFIES-2020}. 
\cref{SI-sec:gen-rel} adds dataset-specific baselines where evaluated, including MiCaM~\cite{MiCaM-ICLR2023}, MoFlow~\cite{Moflow-SIGKDD2020}, EDP-GNN~\cite{EDPGNN-2020}, GraphEBM~\cite{GraphEBM2021}, GDSS~\cite{GDSS+ICML2022}, SMILES-VAE and MHG-VAE~\cite{MHG-VAE-ICML2019}. 
The sequence VAE controls are reported on GuacaMol and RingDiv300k, the two main larger-scale benchmarks, whereas MHG-VAE is reported on QM9 and ZINC250k because MHG induction scales poorly to larger collections.
HGR-VAE and HGR-LDF use the RSG-induced grammar unless otherwise stated. Baselines used released checkpoints or published hyperparameters where available; for new settings such as RingDiv300k, methods were re-tuned within method-specific search spaces.

Generated molecules are assessed along three axes: basic metrics (validity, uniqueness and novelty); chemistry-based distributional metrics, including Fr\'echet ChemNet Distance (FCD), KL score and physicochemical-property agreement; and topology-based metrics covering graph-kernel similarity, nearest-neighbour similarity, scaffold distributions and curvature-filtration descriptors~\cite{CurvFilt+NeurIPS2023}.
These curvature-filtration descriptors are reported as auxiliary distributional checks because of their high estimation variance (\cref{SI-sec:gen-metrics}).
All metrics except validity are computed over valid generated molecules relative to a held-out test set. Additional metrics, full definitions and implementation details are provided in \cref{SI-sec:gen-metrics}. 
Metric computation was also optimised to reduce evaluation overhead; representative speed-ups are shown in \cref{SI-fig:index_time_opt}.

\paragraph{Foundation-model baselines and transfer protocols.}
For downstream transfer, we used two protocols. 
Full fine-tuning jointly optimised the pretrained encoder and task-specific head.
Probing froze the encoder and trained only lightweight task-specific heads, isolating representation quality while limiting overfitting on small downstream tasks. 
Fingerprints from the frozen encoder were cached and reused across probing runs.
Baselines included representative molecular graph pretraining methods: Infomax~\cite{Informax-ICLR2019}, EdgePred~\cite{graphsage-edgepred-2017}, AttrMasking and ContextPred~\cite{ContextPred-AttrMasking-ICLR2020}, GraphCL~\cite{GraphCL-NeurIPS2020}, GraphLoG~\cite{GraphLoG-ICML2021}, GraphMAE~\cite{GraphMAE-KDD2022}, GraphMVP~\cite{GraphMVP-ICLR2022} and MORE~\cite{MORE-AAAI2025}. Transfer performance was measured using the area under the receiver operating characteristic curve (AUC), following the MoleculeNet convention.

\paragraph{Implementation details.}
Generation models (HGR-VAE, HGR-LDF) and HGR-FM were implemented in PyTorch. 
Complete model configurations, optimisers and training hyperparameters are listed in \cref{SI-sec:hyperparameters}.
Unless otherwise stated, results are computed over three random seeds. 
Error bars in \cref{fig:gen-rel} denote the standard deviation across seeds. 
For the foundation-model transfer experiments, \cref{fig:FM-rel} reports the mean across seeds, and per-task standard deviations are provided in \cref{SI-tab:FM_transfer_combined}.

\section*{Data availability}
RingDiv, RingDiv300k and other supporting data generated in this study will be publicly released upon publication.
Public benchmark datasets used in this study are available from their original sources.

\section*{Code availability}
The code used in this study is available through the \href{https://github.com/circle-group/HGR}{HGR GitHub repository}.

\section*{Acknowledgements}
We thank Odin Zhang for his valuable feedback and insightful discussions.
T. Birdal was supported by a UKRI Future Leaders Fellowship [grant number MR/Y018818/1].

\section*{Author contributions statement}
Y.H. conceived the study. 
Y.H., Y.Z., V.P.D. and T.B. designed the study. 
Y.H. and Y.Z. performed the research. 
J.W. collected data and performed cheminformatics analyses.
Y.H., V.P.D., J.L. and T.B. analysed the results. 
S.F. led the visualisation.
All authors contributed to writing and reviewing the paper.

\section*{Competing interests}
The authors declare no competing interests.

\clearpage
\begin{center}
{\Large \textbf{Supplementary Information}}
\end{center}

\setcounter{section}{0}
\setcounter{figure}{0}
\setcounter{table}{0}
\renewcommand{\thesection}{\arabic{section}}
\renewcommand{\thefigure}{\arabic{figure}}
\renewcommand{\thetable}{\arabic{table}}
\renewcommand{\figurename}{Supplementary Fig.}
\renewcommand{\tablename}{Supplementary Table}
\crefalias{section}{suppsec}
\crefalias{subsection}{suppsec}
\crefalias{subsubsection}{suppsec}
\crefalias{figure}{suppfig}
\crefalias{table}{supptab}

\ifdefined\startcontents
  \startcontents[sup]

  \titlecontents{section}[0em]{\addvspace{0.6pc}\bfseries}{\thecontentslabel\quad}{}{\titlerule*[0.8pc]{.}\contentspage}

  \titlecontents{subsection}[2em]{\addvspace{0.3pc}}{\thecontentslabel\quad}{}{\titlerule*[0.8pc]{.}\contentspage}

  \begingroup
    \value{tocdepth}=2
    \printcontents[sup]{}{1}{\section*{Supplementary Contents}}
  \endgroup
\fi

\clearpage

\section{Model and algorithm details}

\subsection{Notation and terminology}
We begin by summarising the key concepts and notation used in this paper.

\begin{table}[h!]
\centering
\caption{Notation summary.}
\small
\begin{tabularx}{0.95\textwidth}{@{} l X @{}}
\toprule
\textbf{Symbol} & \textbf{Description} \\
\midrule
$\mathcal{G} = (\mathcal{V}, \mathcal{E})$ &
Molecular graph with vertex set $\mathcal{V}$ (atoms) and edge set $\mathcal{E}$ (bonds). \\
$U\sim_{\mathcal{G}}W$ &
Adjacency between vertex subsets: $\exists (u,w)\in\mathcal{E}$ with $u\in U$ and $w\in W$. \\
$\mathcal{C}=(\mathcal{V},\mathcal{X},\rk)$ &
Combinatorial complex with cell family $\mathcal{X}$ (atom and lifted motif cells) and rank function $\rk$; the lifting variant is written $\mathcal{C}_{\eta}$, $\eta\in\{\mathrm{MIG},\mathrm{RSG}\}$, and dropped when clear. \\
$\mathcal{M}$ &
Motif vocabulary mined from the molecule library to score merge candidates in lifting. \\ 
$\Gamma = (\Sigma, \mathcal{N}, S, \mathcal{P})$ &
Higher-order grammar with terminal set $\Sigma$, non-terminal set $\mathcal{N}$, start symbol $S \in \mathcal{N}$, and production-rule corpus $\mathcal{P}$. \\
$N^{*} \in \mathcal{N}$ & A generic non-terminal symbol marking replaceable sites; $N_k^*$ denotes the site marked by $N^*$ and created at parsing step $k$. \\
$\mathcal{P} = \{ r_i \mid i = 1, \cdots, |\mathcal{P}| \}$ &
Production-rule corpus. \\
$\sigma,\ \mathcal{V}_{\mathrm{anc}}^\sigma$ &
Cell being contracted and its anchor set in the original molecular graph. \\
$r_i:\mathcal{G}^L_i\to\mathcal{G}^R_i$ &
Production rule with a left-hand-side interface graph $\mathcal{G}^L_i$ and a right-hand-side local graph $\mathcal{G}^R_i$. We omit the subscript $i$ when referring to a generic rule. \\
$(\mathcal{C}_{m},\mathcal{G}_{m}) \rightarrow \bm{r}^{(m)}= [r_{i_1}, \ldots, r_{i_{\ell_{m}}}] $ &
HGR string obtained by parsing $\mathcal{C}_{m}$ paired with $\mathcal{G}_{m}$. The index $m$ runs over molecules, $k$ indexes positions in a string, and $i_k$ identifies the rule in $\mathcal{P}$ at position $k$. \\
$S \Rightarrow^{\bm{r}^{(m)}} (\mathcal{C}_{m},\mathcal{G}_{m})$ &
Derivation that reconstructs $(\mathcal{C}_{m},\mathcal{G}_{m})$ from $S$ by consuming the reversed sequence $\operatorname{rev}(\bm{r}^{(m)})$. \\
\bottomrule
\end{tabularx}
\label{SI-tab:notations}
\end{table}

We use the following grammar-operation terminology throughout the Supplementary Information.
Let $\mathcal{D}$ denote the collection of input data.
\textbf{Induction} constructs a higher-order grammar $\Gamma$ from $\mathcal{D}$.
A \textbf{rewriting} or \textbf{rule-application} step, written $A \Rightarrow B$, replaces one non-terminal site in $A$ with the right-hand side of an applicable production rule, yielding $B$.
We write $A \Rightarrow^* B$ when $B$ is obtained from $A$ by a finite sequence of such rule applications; this process is called a \textbf{derivation}.
In HGR, derivation proceeds in the generative direction.
An HGR string $\bm{r}^{(m)}$ is stored in post-order, so reconstruction from $S$ consumes $\operatorname{rev}(\bm{r}^{(m)})$. We use $S \Rightarrow^{\bm{r}^{(m)}}(\mathcal{C}_{m},\mathcal{G}_{m})$ as shorthand for the derivation represented by $\bm{r}^{(m)}$.
For a target lifted complex and associated molecular graph $(\mathcal{C},\mathcal{G})\in\mathcal{D}$, \textbf{parsing} is the inverse task of finding an HGR string $\bm{r}$ such that $S \Rightarrow^{\bm{r}}(\mathcal{C},\mathcal{G})$.

\subsection{Lifting algorithms}
\label{SI-sec:lifting-alg}

The MIG lifting procedure, summarised in Methods and detailed in \cref{SI-alg:mig_lifting}, performs hierarchical subgraph merging with anchor resolution.
The algorithm initialises singleton atom cells and a current partition over the molecular graph $\mathcal{G}$.
At each round, it scores adjacent pairs of current partition blocks by frequency, commits a maximal non-overlapping set of highest-scoring merges, and appends the resulting merged cells to the merge history $\mathcal{H}$.
The lifted complex $\mathcal{C}_{\mathrm{MIG}}$ is formed as $(\mathcal{V},\operatorname{cells}(\mathcal{H}),\rk_{\mathrm{MIG}})$, where $\operatorname{cells}(\mathcal{H})$ collects the cells recorded in $\mathcal{H}$, with ranks given by cell height above the atom leaves; $\mathcal{G}$ supplies adjacency for candidate generation and anchor resolution.
If the call to \textsc{ResolveAmbiguity} is omitted, \cref{SI-alg:mig_lifting} reduces to the shared frequency-guided merging step.
With this subroutine, it becomes the MIG lifting variant: whenever a candidate merged cell has more than two anchors, the procedure grows the cell by best-first absorption of neighbouring partition blocks until the anchor interface has at most two anchors.
The RSG variant uses the same primitives, but stages admissible merges around ring systems (Methods), so we do not repeat its pseudocode.

\begingroup
\normalcolor
\begin{algorithm}[!ht]
\caption{MIG lifting}
\label{SI-alg:mig_lifting}
\textbf{Input:}
  Molecular graph $\mathcal{G}=(\mathcal{V},\mathcal{E})$; motif vocabulary $\mathcal{M}$ with library frequencies $\mathrm{freq}(\cdot)$\\
\textbf{Return:}
  Lifted complex $\mathcal{C}_{\mathrm{MIG}}$
\begin{algorithmic}[1]
\Procedure{HierarchicalMerge}{$\mathcal{G},\mathcal{M}$}
  \State $\mathcal{B}\gets\{\,\{v\}\mid v\in\mathcal{V}\,\}$;\quad
         $\mathcal{H}\gets[\,\{v\}\mid v\in\mathcal{V}\,]$;\quad
         $\rk_{\mathrm{MIG}}(\{v\})\gets0$ \Comment{Initial partition blocks, trace and ranks}
  \While{there exist $B_i,B_j\in\mathcal{B}$ with $B_i\sim_{\mathcal{G}}B_j$} \Comment{Continue while a merge is possible}
    \State $\mathcal{U}_{\mathrm{raw}}\gets\{\,B_i\cup B_j:B_i,B_j\in\mathcal{B},\ B_i\sim_{\mathcal{G}}B_j\,\}$ \Comment{Raw adjacent merge candidates}
    \State $\mathcal{U}_{\mathrm{res}}\gets\{\,\textsc{ResolveAmbiguity}(\mathcal{G},U,\mathcal{B},\mathcal{M}):U\in\mathcal{U}_{\mathrm{raw}}\,\}$ \Comment{Anchor-resolved candidates}
    \State $\mathcal{U}_{\mathrm{sel}}\gets$ a maximal non-overlapping subset of the highest-frequency candidates in $\mathcal{U}_{\mathrm{res}}$
    \ForAll{$U^*\in\mathcal{U}_{\mathrm{sel}}$}
      \State $\rk_{\mathrm{MIG}}(U^*)\gets1+\max\{\,\rk_{\mathrm{MIG}}(B):B\in\mathcal{B},\ B\cap U^*\ne\emptyset\,\}$
      \State $\mathcal{B}\gets\{B\in\mathcal{B}:B\cap U^*=\emptyset\}\cup\{U^*\}$ \Comment{Update current partition}
      \State $\mathcal{H}\gets\mathrm{Append}(\mathcal{H},U^*)$ \Comment{Record the lifted cell}
    \EndFor
  \EndWhile
  \State \Return $\mathcal{C}_{\mathrm{MIG}}=(\mathcal{V},\operatorname{cells}(\mathcal{H}),\rk_{\mathrm{MIG}})$
\EndProcedure
\Statex
\Procedure{ResolveAmbiguity}{$\mathcal{G},U_0,\mathcal{B},\mathcal{M}$}
  \State $\operatorname{Anc}_{\mathcal{G}}(U):=\{v\in\mathcal{V}\setminus U:\exists u\in U\ \text{s.t.}\ (u,v)\in\mathcal{E}\}$.
  \State initialise max-heap $\mathsf{PQ}$ with $U_0$, keyed by $\mathrm{freq}(\cdot)$;\quad
         $\mathsf{Visited}\gets\{U_0\}$ \Comment{Best-first frontier}
  \While{$\mathsf{PQ}\ne\emptyset$}
    \State $U\gets\textsc{PopMax}(\mathsf{PQ})$
    \If{$|\operatorname{Anc}_{\mathcal{G}}(U)|\le2$}
      \State \Return $U$ \Comment{Resolved interface}
    \EndIf
    \State $\mathcal{Q}_{U}\gets\{\,U\cup B:B\in\mathcal{B},\ B\cap U=\emptyset,\ B\sim_{\mathcal{G}}U\,\}\setminus\mathsf{Visited}$
    \State insert all $U'\in\mathcal{Q}_{U}$ into $\mathsf{PQ}$;\quad
           $\mathsf{Visited}\gets\mathsf{Visited}\cup\mathcal{Q}_{U}$
  \EndWhile
\EndProcedure
\end{algorithmic}
\end{algorithm}
\endgroup

\subsection{Grammar induction}

Grammar induction constructs a shared HGR corpus by parsing lifted complexes and their associated molecular graphs into HGR strings. Here we describe the contraction hierarchy, parsing algorithm, library-level rule induction, reverse derivation and reconstruction guarantees.
In this construction, $\mathcal{C}$ supplies the lifted cell hierarchy, whereas $\mathcal{G}$ supplies the pairwise bonds and anchor interfaces required for rule construction.
For the cell families produced by our lifting algorithms, any two cells are either disjoint or nested, so set inclusion defines a contraction tree. Singleton atom cells are leaves, and a larger cell is the parent of a contained cell when no intermediate cell exists between them in $\mathcal{X}$.
The rank $\rk(\sigma)$ measures the height of $\sigma$ above the atom leaves and is defined as the maximum number of edges from $\sigma$ to an atom leaf. It is therefore determined by $\mathcal{X}$.
We serialise all cells, including singleton atom cells, in post-order. We call a serialisation \emph{admissible} when each child subtree forms a contiguous block before its parent; sibling subtrees may occur in any order.
Different admissible serializations may produce alternative HGR strings for the same lifted complex. Each string remains lossless under the corresponding reverse derivation and may serve as optional data augmentation.

\paragraph{Parsing and rule induction.}
The parsing procedure for a single lifted complex is summarised in \cref{SI-alg:parse_complex}. For the current working graph $\widetilde{\mathcal{G}}$ and a vertex set $U$, $\widetilde{\mathcal{G}}[U]$ denotes the subgraph induced by $U$.
The hatted graphs $\widehat{\mathcal{G}}^L_k$ and $\widehat{\mathcal{G}}^R_k$ are the candidate left- and right-hand sides constructed at parsing step $k$. \textsc{LookupOrInsert} maps this pair to a compatible corpus rule $r_{i_k}:\mathcal{G}^L_{i_k}\to\mathcal{G}^R_{i_k}$, inserting the candidate if necessary, and returns its index $i_k$.
Compatibility requires structural isomorphism that respects the ordered anchor interface and child-site correspondence.
Across a molecule library, \textsc{ParseComplex} is applied to each lifted complex and its associated molecular graph.

\begin{figure}
    \centering
    \includegraphics[width=1\linewidth]{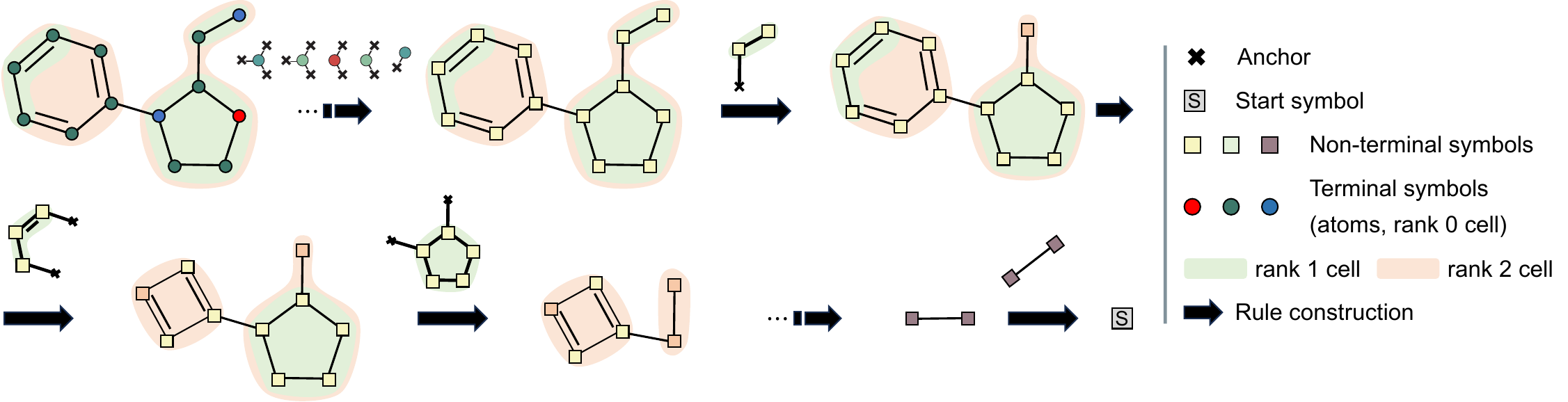}
\caption{ \textbf{Illustration of the grammar induction process.}
Cells are successively contracted into non-terminal sites, shown as squares, until only the start symbol $S$ remains. Applying the recorded rules in reverse order reconstructs the lifted complex and its associated molecular graph, whereas grammar-constrained derivation supports generation.
}
\label{SI-fig:grammar_induction}
\end{figure}

\begin{algorithm}[ht]
\caption{\textsc{ParseComplex}: bottom-up parsing of a lifted complex into an HGR string}
\label{SI-alg:parse_complex}
\textbf{Input:} Lifted complex $\mathcal{C}=(\mathcal{V},\mathcal{X},\rk)$;\ molecular graph $\mathcal{G}=(\mathcal{V},\mathcal{E})$;\ rule corpus $\mathcal{P}$\\
\textbf{Return:} HGR string $\bm{r}$;\ updated corpus $\mathcal{P}$
\begin{algorithmic}[1]
\State $\bm{r}\gets[\,]$
\State $\widetilde{\mathcal{G}}\gets\mathcal{G}$ \Comment{Working copy of $\mathcal{G}$; $\mathcal{X}$ and $\rk$ remain fixed}
  \State choose a post-order serialisation $(t_1,\dots,t_L)$ of the contraction tree \Comment{Sibling-subtree order may vary}
\For{$k=1$ to $L$}
  \State $\sigma\gets\mathrm{cell}(t_k)$
  \State $\widetilde{\sigma}_k\gets$ current representation of $\sigma$ in $\widetilde{\mathcal{G}}$
  \State $\mathcal{V}_{\mathrm{anc}}^{\sigma}\gets\{\,v\in\mathcal{V}\setminus\sigma:\exists u\in\sigma,\ (u,v)\in\mathcal{E}\,\}$ \Comment{Ordered anchors}
  \State $\widehat{\mathcal{G}}^R_k\gets\widetilde{\mathcal{G}}\bigl[\widetilde{\sigma}_k\cup\mathcal{V}_{\mathrm{anc}}^{\sigma}\bigr]$ \Comment{Candidate RHS with ordered child sites}
  \State let $N_k^{*}$ be a fresh non-terminal site
  \State $\widehat{\mathcal{G}}^L_k\gets(\{N_k^{*}\}\cup\mathcal{V}_{\mathrm{anc}}^{\sigma},\ \{(N_k^{*},v):v\in\mathcal{V}_{\mathrm{anc}}^{\sigma}\})$ \Comment{Candidate LHS}
  \State $i_k\gets\textsc{LookupOrInsert}\bigl(\mathcal{P},(\widehat{\mathcal{G}}^L_k\!\to\!\widehat{\mathcal{G}}^R_k)\bigr)$ \Comment{Corpus-rule index}
  \State $\bm{r}\gets\bm{r}\mathbin{\|}[r_{i_k}]$ \Comment{$\|$ denotes sequence concatenation}
  \State contract $\widetilde{\sigma}_k$ to $N_k^{*}$ in $\widetilde{\mathcal{G}}$ \Comment{Preserve boundary incidences at the ordered anchors}
\EndFor
\State mark the root site $N_L^{*}$ with the start symbol $S$
\State \Return $\bm{r},\mathcal{P}$
\end{algorithmic}
\end{algorithm}

After all complexes have been parsed, the shared corpus $\mathcal{P}$ defines the grammar $\Gamma=(\Sigma,\mathcal{N},S,\mathcal{P})$. Each molecule is assigned an HGR string $\bm{r}^{(m)}$, whose rule tokens refer to rules in $\mathcal{P}$.
Multiple cell occurrences can therefore contribute tokens that refer to the same production rule. Anchor resolution during lifting further reduces rule fragmentation.
Corpus growth and rule reuse are quantified on ZINC250k (Fig.~2\textbf{b},\textbf{c}) and QM9 (\cref{SI-fig:qm9_grammar_statistics}).
Each HGR string and its referenced rules define a lossless inverse process.

\paragraph{Reverse derivation and exact reconstruction.}
Given an HGR string $\bm{r}=[r_{i_1},\dots,r_{i_L}]$, reverse derivation starts from $S$ and consumes $\operatorname{rev}(\bm{r})=[r_{i_L},\dots,r_{i_1}]$. The decoder maintains a LIFO stack of pending non-terminal sites, initialised as $[S]$. At reverse step $k=L,\dots,1$, it pops the next site and applies the corpus rule $r_{i_k}$. The ordered anchor interface reconnects boundary bonds, after which child non-terminal sites are pushed in reverse of their stored order. Decoding succeeds only if the stack is empty after $r_{i_1}$ has been applied.

Exact reconstruction relies on the ordered rule sequence and, within each rule, the right-hand-side graph with its ordered child non-terminal sites and the ordered anchor interface. Together, these records specify the derivation tree, cell hierarchy and boundary correspondence. The rank function need not be stored separately because it is recovered as cell height in the contraction tree.
For this guarantee, the reconstructed object is the pair $(\mathcal{C},\mathcal{G})$, where $\mathcal{C}$ is the lifted complex and $\mathcal{G}$ is the associated molecular graph.
For an admissible serialisation $\tau$, let $\operatorname{Parse}_{\Gamma,\tau}(\mathcal{C},\mathcal{G})$ denote the HGR string returned by \textsc{ParseComplex}. Let $\operatorname{Der}_{\Gamma}(\bm{r})$ denote the pair produced by reverse derivation.
For claims on isomorphism classes, atom order is fixed by the same canonical molecular labelling used for serialisation; otherwise the statement is read for labelled complexes.

\begin{proposition}[Exact reconstruction]
\label{SI-prop:lossless}
Let $\eta\in\{\mathrm{MIG},\mathrm{RSG}\}$ denote either lifting variant.
For any admissible serialisation $\tau$, reverse derivation exactly recovers any lifted complex $\mathcal{C}_\eta$ and its associated molecular graph $\mathcal{G}$ after parsing by \textnormal{\textsc{ParseComplex}},
\begin{equation*}
\operatorname{Der}_{\Gamma}\!\left(\operatorname{Parse}_{\Gamma,\tau}(\mathcal{C}_\eta,\mathcal{G})\right)=(\mathcal{C}_\eta,\mathcal{G}).
\end{equation*}
\end{proposition}

\begin{proof}
At parsing step $k$, \textsc{ParseComplex} appends the rule token $r_{i_k}$ and replaces $\widetilde{\sigma}_k$ with a fresh non-terminal site $N_k^*$. Let $(\mathcal{C}^{(k)},\mathcal{G}^{(k)})$ represent the intermediate parsing state after this contraction. Initially, $(\mathcal{C}^{(0)},\mathcal{G}^{(0)})=(\mathcal{C},\mathcal{G})$, whereas the final state $(\mathcal{C}^{(L)},\mathcal{G}^{(L)})$ is represented by $S$. Let $(\widehat{\mathcal{C}}^{(k)},\widehat{\mathcal{G}}^{(k)})$ denote the decoder state after applying $r_{i_L},\dots,r_{i_{k+1}}$ in reverse derivation.

The ordered rule tokens and child non-terminal sites determine the derivation tree, while the ordered anchor interfaces fix the boundary correspondences. If the parsing and decoder states agree after step $k$, applying $r_{i_k}$ restores exactly the cell representation, child sites, and bonds removed by the $k$th contraction. Therefore,
\begin{equation*}
(\widehat{\mathcal{C}}^{(k)},\widehat{\mathcal{G}}^{(k)})
=(\mathcal{C}^{(k)},\mathcal{G}^{(k)})
\Longrightarrow
(\widehat{\mathcal{C}}^{(k-1)},\widehat{\mathcal{G}}^{(k-1)})
=(\mathcal{C}^{(k-1)},\mathcal{G}^{(k-1)}).
\end{equation*}
The final parsing state and initial decoder state at $k=L$ are both represented by $S$, establishing the induction base.
Downward induction over $k=L,\dots,1$ therefore restores the cell-inclusion hierarchy and the associated molecular graph. Because $\rk(\sigma)$ is defined as the height of $\sigma$ in the contraction tree, the restored tree uniquely determines $\rk$. Hence, $(\widehat{\mathcal{C}}^{(0)},\widehat{\mathcal{G}}^{(0)})=(\mathcal{C},\mathcal{G})$, completing the reconstruction.
\end{proof}

\paragraph{Grammar-constrained generation and validity.}
For generation, the reverse-derivation decoder predicts rules in the order corresponding to $\operatorname{rev}(\bm{r})$, beginning with the root rule. At each step, the stack top selects the pending non-terminal site to expand. The admissibility mask retains only rules whose left-hand-side ordered anchor interface matches the interface at that site. After expansion, newly introduced child non-terminal sites are pushed onto the stack. Reversing the completed output recovers the post-order HGR string. A sequence is accepted only when all selected rules have been applied and the stack is empty. Every cell and bond in the resulting object then comes from an applied rule, while cell ranks follow from the completed derivation tree. This establishes grammar-level completion; the chemical scope of ``valid by construction'' requires the rule-origin condition stated next.

Suppose that every applied rule was induced from a valence-consistent molecular graph. Each rule records its right-hand-side graph, bond attributes and ordered anchor interface. Derivation therefore introduces only atoms, bonds and local attachment patterns recorded by these rules. Restricting rule application to a matching ordered anchor interface preserves the local valence and bonding constraints encoded by the grammar. Under these conditions, every completed derivation is grammar-valid by construction. This guarantee does not extend to properties absent from the grammar, including stereochemical realizability and synthesizability, which remain empirical outcomes.

\subsection[Structure-aware rule initialisation]{Structure-aware rule initialisation}
\label{SI-sec:rule-init}

In the generation model, each rule token is initialised from a combination of Weisfeiler–Lehman (WL) structure-aware features and a learned per-rule embedding. For the molecular foundation model, we use a fully structure-derived initialisation that combines these WL features with a complementary Random Walk Positional Encoding (RWPE).

\paragraph{Weisfeiler–Lehman histogram embedding.}

For each production rule $r_i \in \mathcal{P}$, we construct a fixed-length, order-invariant feature vector by aggregating WL labels over its right-hand-side graph $\mathcal{G}^R_i=(V_i, E_i)$. The WL refinement proceeds for $H$ iterations.

At the initialisation step ($t = 0$), each node is assigned a canonical attribute label derived from stable atomic properties (e.g., symbol, aromaticity, degree, ring membership) and structural markers (e.g., anchor flags). These tokens are lexicographically sorted and concatenated to ensure permutation invariance.
In subsequent rounds $(t = 1, \dots, H-1)$, node labels are iteratively refined by concatenating each node’s previous label with the sorted multiset of its neighbours’ labels, following the standard 1-WL refinement scheme.

We define a label universe $\Lambda^{(t)} = \{\lambda^{(t)}_1, \dots, \lambda^{(t)}_{|\Lambda^{(t)}|}\}$ for each iteration $t$, which contains all distinct WL labels observed across rules and an $\langle \text{UNK}\rangle$ to handle unseen structures during inference.
Each rule $r_i$ is then represented by a normalised histogram of label occurrences:
\begin{equation*}
h_{i,j}^{(t)} = \frac{1}{|V_i|} \sum_{v \in V_i} \mathds{1}\bigl\{\ell_v^{(t)} = \lambda_j^{(t)}\bigr\},
\quad j = 1, \dots, |\Lambda^{(t)}|,
\end{equation*}
where $\ell_v^{(t)}$ denotes the WL label assigned to node $v$ in rule $r_i$ at refinement round $t$, and $\mathds{1}\{\cdot\}$ is the indicator function.
Stacking all $h_i^{(t)}$ across rules yields a sparse matrix $ \mathbf{S}^{(t)} \in \mathbb{R}^{|\mathcal{P}| \times |\Lambda_t|}$,
whose $i$-th row corresponds to the histogram of rule $r_i$.

The cumulative WL feature for rule $r_i$ is obtained by concatenating the histograms from all refinement rounds:
\begin{equation*}
\mathbf{z}_i^{\mathrm{WL}} = \bigl[ \mathbf{S}_i^{(0)} \mathbin{\|} \mathbf{S}_i^{(1)} \mathbin{\|} \dots \mathbin{\|} \mathbf{S}_i^{(H-1)} \bigr],
\end{equation*}
where $\mathbf{S}_i^{(t)}$ denotes the $i$-th row of $\mathbf{S}^{(t)}$.
This embedding is invariant to node and neighbourhood ordering, supports sparse storage and large-scale GPU computation, and extends to previously unseen rules.
In our grammars, WL embedding already yields distinct initial features for more than $98\%$ of production rules, indicating strong identifiability.

\paragraph{Random walk-based rule features.}
To further increase the discriminative power of rule initialisation beyond the WL histogram, we augment each production rule with a random walk positional encoding (RWPE)~\cite{RWPE-ICLR2022, HoRW-IS2024} computed on the rule’s right-hand side graph.

For each rule $r_i$, we first extract its adjacency matrix $A\in\mathbb{R}^{|V_i|\times |V_i|}$ based on its RHS graph $\mathcal{G}^R_i=(V_i, E_i)$. We then construct a random walk transition matrix
\begin{equation*}
P = D^{-1}A,\quad D_{ii}=\sum_j A_{ij},
\end{equation*}
where $D$ is the degree matrix.
The structural signature of a node $v$ is defined as the vector of return probabilities over $K$ steps:
\begin{equation*}
\mathbf{p}_v = \big[(P)_{vv}, (P^2)_{vv}, \ldots, (P^K)_{vv}\big]\in\mathbb{R}^K.
\end{equation*}
We compute these values iteratively via $P^{t+1} = P^t P$ and assemble them into a node feature matrix $\mathbf{R} \in \mathbb{R}^{n \times K}$, where the $v$-th row corresponds to $\mathbf{p}_v$.

To ensure invariance to node permutation, we impose a canonical ordering on the nodes by performing a lexicographical sort on the rows of $\mathbf{R}$. To map graphs of varying sizes to a fixed-dimensional space, we truncate or pad the sorted matrix to a maximum size $N_{\max}$. The processed matrix is then flattened to form the rule-level RWPE vector:
\begin{equation*}
\mathbf{z}_i^{\mathrm{RW}} = \mathrm{vec}\Bigl(\mathrm{PadOrTrunc}\bigl(\mathrm{LexSort}(\mathbf{R}); N_{\max}\bigr)\Bigr) \in \mathbb{R}^{N_{\max} K}.
\end{equation*}

This encoding relies solely on graph topology, providing a complementary signal to the attribute-based WL features. The final initialised embedding $\mathbf{z}_i$ for rule $r_i$ is the concatenation of both components:
\begin{equation*}
\mathbf{z}_i = \bigl[ \mathbf{z}_i^{\mathrm{WL}} \mathbin{\|} \mathbf{z}_i^{\mathrm{RW}} \bigr].
\end{equation*}
The inclusion of $\mathbf{z}_i^{\mathrm{RW}}$ provides an additional topological signal that can distinguish structurally distinct rules with identical WL colourings.

\subsection[HGR-based de novo generation]{HGR-based de novo generation}

\begin{figure}
\centering
\includegraphics[width=1\linewidth]{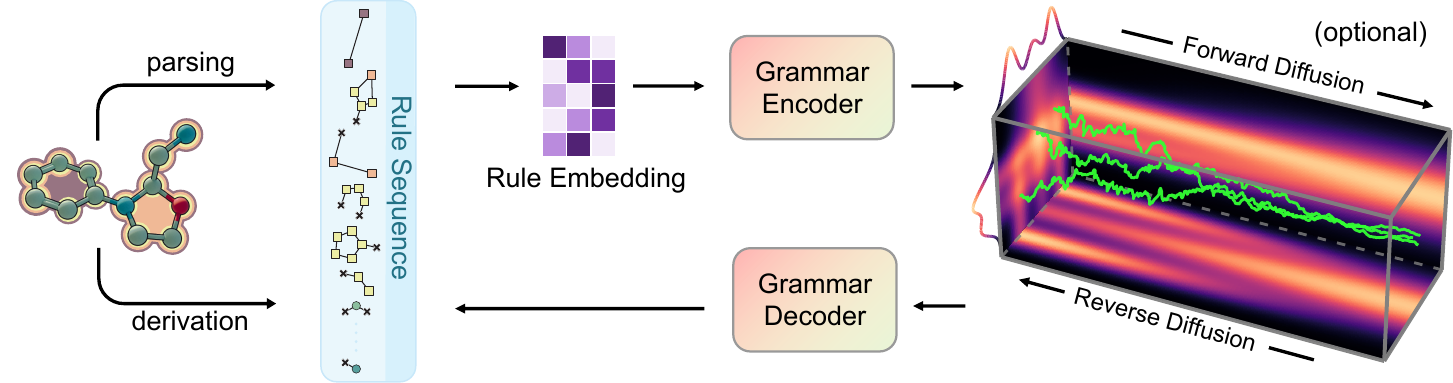}
\caption{\textbf{Workflow of HGR-based molecular generative models.}
Molecules are parsed into HGR strings, encoded by the grammar encoder and reconstructed by the grammar decoder. An optional latent diffusion module enables sampling in the learned latent space.
}
\label{SI-fig:gen-framework}
\end{figure}

\subsubsection{Grammar variational autoencoder}
\label{SI-sec:gvae-arch}

The HGR generative model is a sequence-to-sequence grammar variational autoencoder (GVAE) that operates on HGR strings rather than atoms (\cref{SI-fig:gen-framework}).

For all reported HGR-VAE and HGR-LDF results, each rule token is mapped to a WL-structured embedding, and a Transformer encoder maps the resulting variable-length sequence to a fixed-dimensional summary.
This representation parameterises the Gaussian posterior $q_{\bm{\phi}}(\bm{z}\mid\bm{r})$, from which a latent code $\bm{z}\in\mathbb{R}^{d}$ is sampled using the reparameterization trick.

The decoder is a $\bm{z}$-conditioned GRU that autoregressively generates production rules, with grammar constraints applied at each decoding step.
A last-in-first-out stack determines which pending non-terminal site is expanded next, while anchor-interface matching restricts rule selection to admissible production rules. 
During decoding, rule generation is capped at the dataset-specific maximum sequence length. Sequences whose stacks remain non-empty at this limit are discarded as incomplete derivations; completed HGR strings are syntactically valid and decode deterministically to molecular complexes.
The model is trained with the $\beta$-VAE objective defined in Methods under a cyclical Kullback--Leibler annealing schedule.
The network widths, depths and per-dataset training settings are listed in \cref{SI-tab:gvae_config}.

\subsubsection[Latent diffusion sampling]{Latent diffusion sampling}
\label{SI-sec:latent-diffusion}

To further improve sample quality, we model the GVAE latent space with a score-based diffusion framework that perturbs latent codes with noise and trains a neural network to reverse this corruption.
Sampling then starts from a tractable prior and follows the learned reverse process to generate latent codes approximating the target distribution $p(\bm{z}_0)$~\cite{NCSN+NeurIPS2019, Score-SDE+ICLR2021}.

Specifically, the time-dependent forward process of the diffusion model can be described by the following stochastic differential equation (SDE):
\begin{equation*}
\mathrm{d}\bm{z}_t=\bm{f}_t\left(\bm{z}_t\right)\mathrm{d}t+g_t\mathrm{d}\bm{w}_t,
\end{equation*}
where $\bm{f}_t: \mathbb{R}^n \to \mathbb{R}^n$ is a vector-valued drift function, $g_t: [0,T]\to \mathbb{R} $ is a scalar diffusion coefficient, and $\bm{w}_t$ represents a Wiener process. 
Under this SDE, $p(\bm{z}_0)$ evolves to $p(\bm{z}_T)$, which for sufficiently large $T$ approximates a tractable prior, such as a standard Gaussian distribution.

Starting from time $T$, $p(\bm{z}_T)$ can be progressively transformed back to $p(\bm{z}_0)$ by following the trajectory of the reverse-time SDE~\cite{SDEreverse1982},
\begin{equation*}
    \mathrm{d} \bm{z}_t=[\bm{f}_t(\bm{z}_t)-g_t^2  \underbrace{\nabla_{\bm{z}_t} \log p_t(\bm{z}_t)}_{\text{score function}}]\diff{\bar{t}}
    + g_t \mathrm{d} \bar{\bm{w}}_t,
\end{equation*}
where $p_t(\cdot)$ denotes the probability density function of $\bm{z}_t$ and $\bar{\bm{w}}$ is a reverse-time Wiener process.

The term $\nabla_{\bm{z}_t} \log p_t(\bm{z}_t)$, known as the score function, is typically parameterised by a neural network $\bm{s}_{\bm{\theta}}(\bm{z}_t,t)$ and trained using the conditional score-matching loss function~\cite{Scorematching2011}:
\begin{equation*}
    \ell(\bm{\theta})
\triangleq \mathbb{E}_{t, \bm{z}_t}\left[\omega(t)\left\|\bm{s}_{\bm{\theta}}(\bm{z}_t,t) - \nabla_{\bm{z}_t} \log p_t(\bm{z}_t)\right\|^2\right]
\propto
\mathbb{E}_{t,\bm{z}_0,\bm{z}_t }\left[ \omega(t) \left \Vert \bm{s}_{\bm{\theta}}(\bm{z}_t, t) - \nabla_{\bm{z}_t} \log p_t (\bm{z}_t|\bm{z}_0)\right\Vert^2\right],
\end{equation*}
where $\omega(t)$ is a weighting function. The second form is more commonly used because the conditional probability $p_t (\bm{z}_t|\bm{z}_0)$ is tractable. Generation then samples $\bm{z}_T$ from a tractable prior $p(\bm{z}_T ) \approx p_{\mathrm{prior}}(\bm{z})$ and recovers $\bm{z}_0$ by numerically solving the reverse-time SDE.

In HGR-LDF the score network $\bm{s}_{\bm{\theta}}$ is a latent diffusion Transformer that operates on the standardised GVAE posterior means while the grammar decoder is left unchanged; its architecture, the variance-preserving noise schedule and the predictor sampler are specified in \cref{SI-tab:ldf_config}.

\subsection{HGR-FM architecture and pretraining}
\label{SI-sec:HGR-FM}

HGR-FM operates on HGR strings and is pretrained on large molecular datasets. 
It combines structure-aware rule embeddings with a Grammar Transformer backbone, a higher-order position encoding and a multi-level pretext objective.

\paragraph{Grammar Transformer backbone.}
The Grammar Transformer is the shared encoder backbone of HGR-FM.
Each rule token is mapped to a structure-aware embedding (\cref{SI-sec:rule-init}) and linearly projected to the model width, with dedicated vectors for padding and special boundary tokens.
A stack of Transformer layers then processes the sequence.
After the final layer, contextualized atom-level token embeddings feed atom-level tasks, whereas the final hidden state of the molecule-level token serves as a fixed-length molecular fingerprint for fragment-, molecule- and 3D-level tasks and downstream transfer.

\paragraph{Higher-order position encoding.}
Within each Transformer layer, self-attention is augmented with a composite higher-order bias, defined in Methods.
This bias combines rank-aware rotary position encoding over cell height, additive biases indexed by grammar-tree distance and discrete relation type, and shortest-path spatial and edge-path biases~\cite{Graphormer-NeurIPS2021} on atom-token pairs.
Attention heads are divided into four equal groups.
Rank-aware RoPE is applied to all heads; one group receives no additive structural bias, whereas the other three receive the grammar-tree, shortest-path and edge-path biases, respectively.
Per-layer learnable scales weight the three additive biases.

\paragraph{Pretext tasks.}
\label{SI-sec:pretext_tasks}

Supervised molecular labels are often task-specific and costly to obtain, whereas large unlabelled molecular libraries are readily available.
HGR-FM therefore uses multi-level pretext objectives to learn transferable representations from unlabelled molecules.
Following the multi-level objective design of MORE~\cite{MORE-AAAI2025}, these tasks cover atom masking, substructure-fingerprint prediction, descriptor regression and conformation-distance prediction.
They provide atom-, fragment-, molecule- and 3D-conformation-level supervision, respectively.
The atom- and fragment-level tasks are adapted to the HGR rule-token input.

\emph{Atom-level task.}
During data collation, 25\% of atoms are sampled for masking, with at least one masked atom per molecule. The atom feature is replaced by a mask token and, when grammar inputs are used, the corresponding rule-sequence position is replaced by the grammar mask token. The target is the original atom type over the atom vocabulary. A one-layer GIN decoder reconstructs this target from contextualized atom-token embeddings using the scaled cosine-error loss~\cite{GraphMAE-KDD2022}. This task supervises local atomic environments.

\emph{Fragment-level task.}
For fragment-level supervision, the model predicts two binary substructure targets: 1{,}024-bit Morgan fingerprints (radius 2)~\cite{morgan_fp2010} and 166-bit Molecular ACCess System (MACCS)~\cite{maccs-durant2002} fingerprints. These targets expose functional-group and local-substructure information not captured by atom masking alone.

\emph{Molecular-level task.}
For molecular-level supervision, the target vector contains standardised RDKit descriptors, including molecular weight and logP, to represent global physicochemical properties.

\emph{3D-conformation-level task.}
For 3D-conformation supervision, optimised conformers provide pairwise interatomic-distance targets, adding geometric constraints to the HGR-string representation.

The resulting weighted pretraining loss uses the per-task coefficients $\lambda$ listed in~\cref{SI-tab:fm_config}. Together, the tasks supervise complementary scales, from individual atoms to whole-molecule geometry.

\section{RingDiv dataset construction}

\subsection{Source databases and curation}
\label{SI-sec:ringdiv-sources}

This section documents the construction of the RingDiv datasets: RingDiv (1.18 million molecules) and RingDiv300k, a curated development subset used for rapid experiments and ablation studies. It covers source aggregation, molecular standardisation, ring-topology annotation and quota-balanced diversity sampling.

The starting pool combined approximately $143$ million compounds from PubChem~\cite{kim2016pubchem}, SureChEMBL~\cite{papadatos2016surechembl}, and five commercial screening libraries: eMolecules, Asinex, ChemDiv, Enamine and MolPort.
All structures were standardised and converted to RDKit molecular objects. Entries that failed RDKit parsing were removed. Source-level duplicates were removed using InChI identifiers~\cite{heller2015inchi} generated from the standardised structures, so repeated records from different repositories were collapsed before filtering.

Ring perception was performed on the standardised pool to obtain ring-size and fusion-topology descriptors. Ring sizes were binned as small (3--4 atoms), normal (5--7 atoms), medium (8--11 atoms) or macro ($\geq 12$ atoms).
Ring systems were also labelled by connectivity pattern as spirocyclic, fused, condensed or bridged.
The annotated pool was filtered to focus on drug-like ring-containing small molecules. We retained compounds with 1--10 rings and molecular weights of 100--900 Da, and excluded molecules containing macrocycles. 
Macrocycles were excluded because they were sparse in the source databases and differ from common ring systems in conformational flexibility and synthetic accessibility.
The resulting per-molecule descriptors were used to guide quota-balanced sampling and to quantify ring-topology coverage in the released benchmark.

\subsection{Quota-balanced diversity sampling}
\label{SI-sec:ringdiv-sampling}

Sampling proceeded in two stages to balance rare ring topologies with whole-molecule structural diversity. 
First, quotas were defined using the ring-system descriptors described above, with emphasis on relatively rare but structurally informative categories such as medium rings, spirocyclic systems and bridged systems.
Molecules were drawn until the minimum quota for each rare class was met, and the remaining sampling budget was filled by random expansion into an oversampled candidate set. This step preserved the intended marginal balance across topology classes before diversity selection.

In the second stage, the oversampled candidates were encoded as 2{,}048-bit Morgan fingerprints (radius 2)~\cite{morgan_fp2010}. 
A MaxMin criterion based on pairwise Tanimoto similarities was then used to select a globally dispersed subset while suppressing near-duplicate neighbours.
We used an approximate Tanimoto-similarity threshold of 0.65, which reduced redundancy without eroding the rare-topology coverage established by the quota pass.

\subsection{RingDiv variants}
\label{SI-sec:ringdiv-variants}

RingDiv is the full released benchmark produced by the curation and sampling procedure above. RingDiv300k is a smaller topology-balanced subset selected from RingDiv, using tighter quotas for rare ring categories while retaining the same drug-like scope.
The two variants therefore share provenance, preprocessing and filtering criteria, but differ in scale and rare-ring balance. Dataset sizes and per-tag molecule counts are given in~\cref{SI-tab:ringdiv_config}.

\begin{table}[h]
\centering
\caption{Ring-topology composition of the released RingDiv variants. Each cell reports the number of molecules containing at least one ring of the indicated type.}
\label{SI-tab:ringdiv_config}
\begin{tabular}{lcccccccc}
\toprule
Variant & $N$ & SRing & NRing & MRing & Spiro & Fused & Condensed & Bridged \\
\midrule
RingDiv     & 1,183,434 & 157,542 & 1,160,912 & 69,061 & 104,198 & 522,461 & 582,747 & 113,423 \\
RingDiv300k & 299,819   & 42,703  & 290,985   & 33,591 & 32,026  & 147,500 & 164,171 & 37,660  \\
\bottomrule
\end{tabular}
\end{table}

\section{Experimental setup}

\subsection{Molecular generation metrics}
\label{SI-sec:gen-metrics}

We categorize the evaluation metrics into three groups: basic metrics, chemistry-based metrics, and topology-based metrics. This last group includes graph-level similarity, nearest-neighbour and scaffold distribution measures, and curvature-filtration descriptors based on persistent homology.

\begin{table}[!ht]
\centering
\caption{Categorization of evaluation metrics for molecular generative models.}
\label{SI-tab:metric_categories}
\begin{tabular}{lll}
\toprule
\multicolumn{2}{l}{\textbf{Category}} & \textbf{Metrics} \\
\midrule
\multicolumn{2}{l}{Basic metrics}
& Validity, Uniqueness, Novelty, V.U.N. \\ \midrule
\multicolumn{2}{l}{Chemistry-based metrics}
& FCD, KL score and SA distribution \\ \midrule
\multicolumn{2}{l}{Topology-based metrics}
& NSPDK, SNN, Scaffold, and curvature-filtration descriptors ($\kappa_{\mathrm{OR}}$, $\kappa_{\mathrm{BF}}$) \\
\bottomrule
\end{tabular}
\end{table}

For QM9, ZINC250k, RingDiv300k and GuacaMol, we evaluated 10{,}000 generated molecules; for MOSES, we used 25{,}000 molecules, matching the benchmark protocols used by the compared methods~\cite{DiGress+ICLR2023}. MOSES and GuacaMol were evaluated with their official benchmarking libraries. 
We follow the MOSES evaluation convention~\cite{moses-2020}: validity is assessed on the raw generated set, whereas all other scores use the valid subset $G$ and compare it with the corresponding reference set $R$.
Unless noted otherwise, chemistry- and topology-based metrics quantify agreement between $G$ and $R$; the preferred direction is marked in the result tables (\cref{SI-sec:gen-rel}).

\begin{itemize}[topsep=2pt,itemsep=2pt,parsep=0pt]
\item \textbf{Validity} reports the proportion of generated structures that pass molecular parsing and sanitisation, including basic valence and bonding checks.

\item \textbf{Uniqueness} measures how many valid outputs remain after canonicalisation and duplicate removal.

\item \textbf{Novelty} is computed after validity filtering and canonicalisation by treating any output that matches a training molecule as a repeat. Small values therefore indicate training-set recall rather than exploration of new chemical structures.

\item \textbf{V.U.N.} combines these three criteria by counting molecules that are valid, unique and novel.

\item \textbf{Fr\'echet ChemNet Distance (FCD)} embeds generated and reference molecules with ChemNet, a neural network trained on bioactivity-related molecular data~\cite{FCD}.
Let $\mu_G,\Sigma_G$ and $\mu_R,\Sigma_R$ denote the empirical means and covariance matrices of the penultimate-layer ChemNet activations for $G$ and $R$. The distance is
\begin{equation*}
\mathrm{FCD}(G,R)
= \left\|\mu_G-\mu_R\right\|_2^2
+ \operatorname{Tr}\left(\Sigma_G+\Sigma_R-2(\Sigma_G\Sigma_R)^{1/2}\right).
\end{equation*}
Lower FCD indicates closer distributions in this learned feature space. We use it as a primary distribution-learning metric because it summarises both location and spread. It remains a Gaussian summary of learned features, so we interpret it together with validity, novelty and structure-level metrics rather than as a standalone measure of generation quality.

\item \textbf{KL score} follows the GuacaMol distribution-learning benchmark and compares descriptor distributions between generated and reference molecules~\cite{guacamol-JCIM2019}.
Descriptor-wise Kullback--Leibler divergences are transformed into a bounded score. 
Higher values indicate better property-distribution agreement.

\item \textbf{Synthetic Accessibility (SA)} is a heuristic estimate of synthetic difficulty~\cite{SAscore-2009}.
When SA is reported with a downward arrow, we report a discrepancy from the reference distribution rather than the raw molecule-level score.

\item \textbf{Neighborhood Subgraph Pairwise Distance Kernel (NSPDK)} evaluates graph-topological agreement through the MMD between generated and test-set molecules, using node and edge features.

\item \textbf{Similarity to a nearest neighbour (SNN)} asks whether each generated molecule has a close analogue in the reference set~\cite{moses-2020}.
We compute it as the average Morgan-fingerprint Tanimoto similarity between each generated molecule and its nearest reference-set neighbour.
SNN is best interpreted jointly with novelty: it checks whether generated molecules stay near the reference-set chemical manifold, but by itself it cannot distinguish faithful local chemistry from nearest-neighbour memorization.

\item \textbf{Scaffold similarity (Scaf)} reduces each molecule set to a Bemis--Murcko scaffold count vector~\cite{BemisMurcko1996}.
Following the MOSES benchmark definition~\cite{moses-2020}, we compare generated and reference scaffold-count vectors by cosine similarity.
Scaf is complementary to SNN because it aggregates core-framework frequencies at the dataset level rather than nearest-neighbour similarity at the molecule level. It can therefore reveal whether a generator preserves the scaffold composition of the reference set, while still allowing molecule-level novelty through different substitutions or side chains.

\item \textbf{Filter score} is reported for MOSES and gives the share of generated molecules retained by the benchmark's medicinal-chemistry and structural filters~\cite{moses-2020}.

\item \textbf{Curvature-filtration descriptors.} We use the curvature-filtration framework~\cite{CurvFilt+NeurIPS2023} to compare graph topology through persistent homology of edge filtrations, complementing the substructure and graph-kernel metrics above.
The framework assigns a scalar curvature to each edge under two complementary notions: Ollivier--Ricci curvature $\kappa_{\mathrm{OR}}$, defined from the $1$-Wasserstein distance between local neighbourhood distributions, and Balanced Forman curvature $\kappa_{\mathrm{BF}}$~\cite{BalForman+ICLR2022}, which depends on endpoint degrees and local cycle counts, including triangles and chordless $4$-cycles. 
Each curvature induces an edge filtration of the molecular graph; we compute its $0$- and $1$-dimensional persistent homology, summarise each diagram as a persistence landscape and compare the landscape distributions of generated and reference sets. 
Because these descriptors showed high sampling variance at the evaluated sample sizes (reference-versus-reference coefficients of variation, roughly 30-60\%), we report them as auxiliary distributional checks rather than primary ranking metrics.

\end{itemize}

\subsection{Functional-group profiling}
\label{SI-sec:fg}

\begin{figure}
    \centering
    \includegraphics[width=1\linewidth]{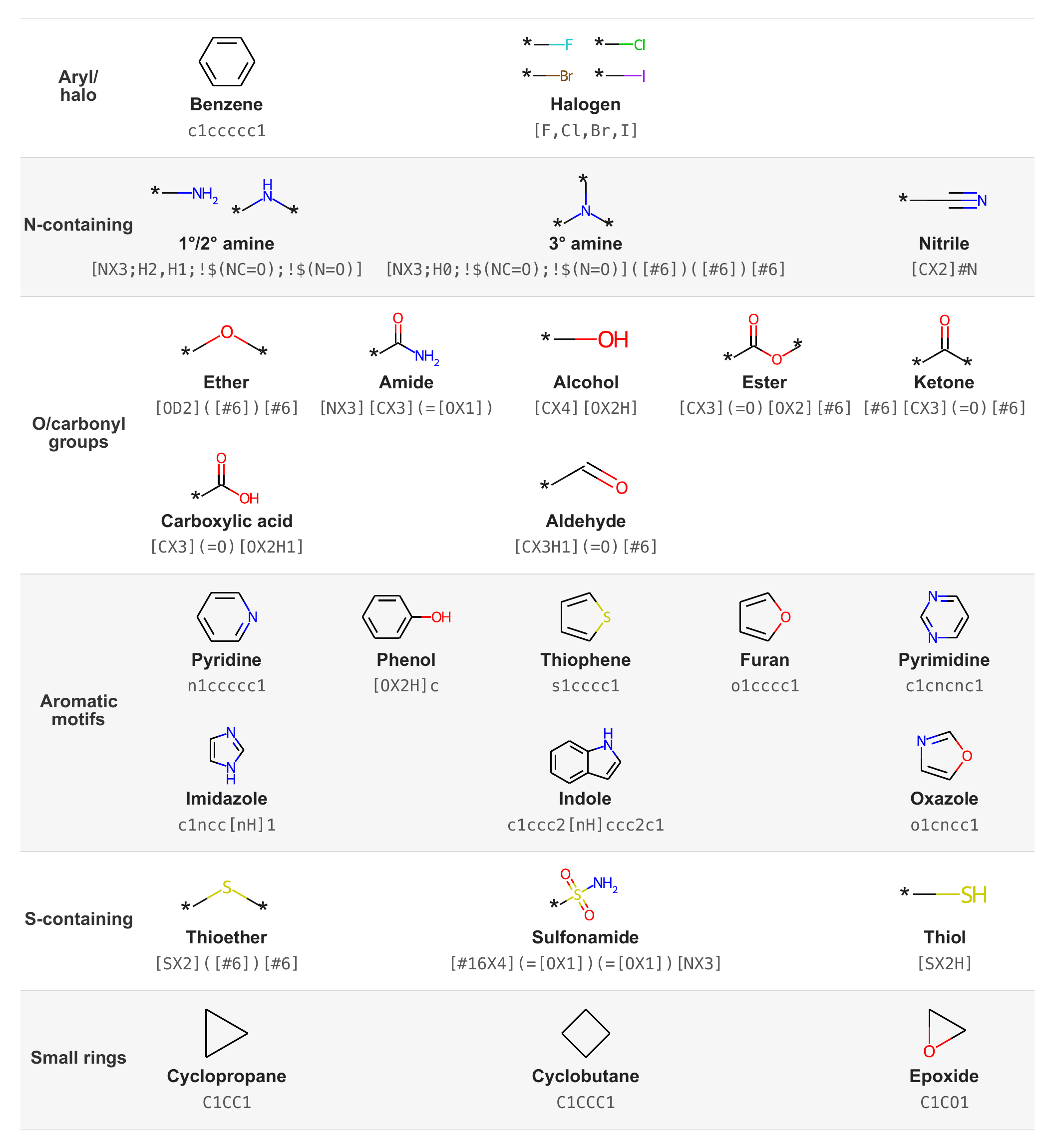}
\caption{{\bf Functional groups and SMARTS patterns for molecular profiling.} Representative structures and RDKit-compatible SMARTS queries used to annotate functional groups in RingDiv300k reference and generated molecules. 
Categories follow a previous functional-group diversity analysis and include aryl and halogen motifs, nitrogen-containing groups, oxygen- and carbonyl-containing groups, aromatic motifs, sulfur-containing groups and strained small rings. 
SMARTS patterns were adapted from prior functional-group profiling work~\cite{YuelDesign-SA2026} and applied as molecule-level, non-exclusive queries.
}
\label{SI-fig:functional-group}
\end{figure}

Following prior work on functional-group diversity~\cite{YuelDesign-SA2026}, we used SMARTS-based substructure profiling to compare the local chemical composition of generated molecules with the RingDiv300k real test set. This analysis complements global fingerprint- and distribution-level metrics by asking whether generated molecules retain recognizable organic functional motifs spanning the categories profiled here.

Reference and generated SMILES were standardised before profiling. We canonicalised SMILES without stereochemical annotations, normalised common functional states, neutralised charged groups where possible, retained the largest fragment, and sanitised the resulting RDKit molecule. Molecules that could not be parsed or sanitised were excluded from this profiling step.

Functional groups were assigned by RDKit substructure matching with SMARTS queries adapted from ref.~\cite{YuelDesign-SA2026}. Each SMARTS query defined one functional-group category, such as carboxylic acid, ester, amide, ketone, aldehyde, primary or secondary amine, tertiary amine, alcohol, phenol, ether, epoxide, thiol, thioether, sulfonamide, halogen, and representative aromatic or special ring motifs. We list the exact SMARTS implementations used in this study in \cref{SI-fig:functional-group}.

For each molecule, a functional group was recorded as present if at least one matching substructure was found, irrespective of the number of matches. Frequencies reported in Fig.~4a,b are therefore molecule-level occurrence rates rather than raw substructure counts. Categories are non-exclusive, so a molecule can contribute to several functional groups.
For readability, the main figure separates common functional groups from less frequent or special motifs. Within each panel, chemically related groups are separated by dashed guides and ordered by their frequency in the real test set.

\subsection{HGR-FM training and evaluation}

We evaluated downstream transfer on seven MoleculeNet benchmarks~\cite{moleculenet-wu2018}, spanning biophysical, physiological and toxicity-related molecular properties.
Biophysical tasks include HIV, which measures HIV replication inhibition; BACE, which measures inhibition of human $\beta$-secretase 1 (BACE-1); and the Side Effect Resource (SIDER), which groups the side effects of marketed drugs into 27 system organ classes.
Physiological and toxicity-related tasks include Blood-brain barrier penetration (BBBP), which predicts blood-brain barrier permeability for molecules targeting the central nervous system. Tox21, ClinTox and ToxCast focus on toxicity-related endpoints.

The baseline-aligned downstream comparison used ZINC2m pretraining, whereas separate pretraining-data analyses used RingDiv and joint ZINC2m$+$RingDiv pretraining.
We used the standard MoleculeNet scaffold split protocol, with 80/10/10 training, validation and test partitions~\cite{MORE-AAAI2025}.
Compared with random splitting, scaffold splitting partitions molecules by core substructure, yielding a more challenging but more realistic evaluation of generalisation.
For downstream transfer, we evaluated two protocols, full fine-tuning and frozen-encoder probing, both with early stopping on validation performance. 
Both protocols used the same shallow three-layer multilayer perceptron (MLP) prediction head. In full fine-tuning, the encoder and head were updated jointly; in probing, the encoder was frozen and only the head was trained. 
We repeated each experiment across three random seeds and report the mean and standard deviation of test performance.
The same hyperparameters were used for both probing and full fine-tuning across all downstream tasks, without task-specific tuning (listed in \cref{SI-tab:fm_config}).

\subsection{Model configurations and training hyperparameters}
\label{SI-sec:hyperparameters}

The tables below summarise the principal architecture and optimisation settings, with continuous values rounded for readability. Complete per-run configurations, including dataset-specific scheduler, diffusion and optimiser settings, are provided with the code release.

The grammar variational autoencoders (\cref{SI-tab:gvae_config}) share a rotary-position Transformer encoder with attention pooling, a GRU decoder, hybrid structural rule initialisation, AdamW optimisation, linear warm-up, and cyclical $\beta$-annealing. 
The MIG-grammar ablation follows the same recipe, differing only in independently searched architecture and optimisation values.

The latent diffusion samplers (\cref{SI-tab:ldf_config}) operate on standardised GVAE posterior means. All use a variance-preserving SDE, Euler--Maruyama sampling without a corrector, final denoising, input LayerNorm and exponential-moving-average training. Dataset-specific diffusion, architecture and optimisation settings are provided in the released configurations. \cref{SI-tab:fm_config} summarises the molecular foundation model.

\begin{table}[p]
\centering
\caption{\textbf{Grammar variational autoencoder (HGR-VAE) configuration} for the five generation benchmarks under the RSG-induced grammar.}
\label{SI-tab:gvae_config}
\footnotesize
\setlength{\tabcolsep}{5pt}
\begin{tabular}{l ccccc}
\toprule
 & QM9 & ZINC250k & RingDiv300k & MOSES & GuacaMol \\
\midrule
\multicolumn{6}{l}{\textit{Architecture}}\\
Max rule-sequence length & 17 & 72 & 137 & 52 & 175 \\
Production rules $|\mathcal{P}|$ & 280 & 403 & 6{,}205 & 320 & 2{,}385 \\
Motif vocabulary & 500 & 500 & 1{,}000 & 1{,}000 & 1{,}000 \\
Latent dimension $d$ & 32 & 256 & 256 & 256 & 768 \\
Rule-token dimension & 64 & 256 & 768 & 512 & 768 \\
Encoder width & 384 & 512 & 512 & 512 & 384 \\
Encoder layers & 3 & 5 & 4 & 4 & 7 \\
Encoder heads & 8 & 8 & 16 & 16 & 8 \\
Decoder width & 384 & 768 & 768 & 512 & 1{,}024 \\
Decoder layers & 3 & 4 & 4 & 4 & 3 \\
Dropout & 0.05 & 0.05 & 0.05 & 0.05 & 0.10 \\
\midrule
\multicolumn{6}{l}{\textit{Optimisation}}\\
Learning rate & $1.8{\times}10^{-3}$ & $4.6{\times}10^{-4}$ & $5.9{\times}10^{-4}$ & $3.3{\times}10^{-4}$ & $1.3{\times}10^{-4}$ \\
Batch size & 1{,}024 & 512 & 512 & 1{,}024 & 256 \\
KL weight $\beta$ & 0.5 & 0.5 & 0.7 & 0.5 & 0.5 \\
\bottomrule
\end{tabular}
\end{table}

\begin{table}[p]
\centering
\caption{\textbf{Latent diffusion (HGR-LDF) configuration.} Tokens $\times$ patch dimension $=d$.}
\label{SI-tab:ldf_config}
\footnotesize
\setlength{\tabcolsep}{6pt}
\begin{tabular}{l cccc}
\toprule
 & ZINC250k & RingDiv300k & MOSES & GuacaMol \\
\midrule
\multicolumn{5}{l}{\textit{Architecture}}\\
Latent dimension $d$ & 256 & 256 & 256 & 768 \\
Tokens & 8 & 16 & 32 & 8 \\
Patch dimension & 32 & 16 & 8 & 96 \\
Hidden width & 768 & 512 & 512 & 1{,}024 \\
Layers & 8 & 8 & 8 & 6 \\
Attention heads & 4 & 8 & 8 & 16 \\
Dropout & 0.10 & 0.05 & 0.05 & 0.10 \\
\midrule
\multicolumn{5}{l}{\textit{Optimisation}}\\
Learning rate & $1.8{\times}10^{-3}$ & $7.5{\times}10^{-4}$ & $8.2{\times}10^{-5}$ & $7.9{\times}10^{-4}$ \\
Batch size & 512 & 1{,}024 & 512 & 512 \\
\bottomrule
\end{tabular}
\end{table}

\begin{table}[t]
\centering
\caption{\textbf{HGR-FM pretraining and fine-tuning configuration.}}
\label{SI-tab:fm_config}
\small
\begin{tabular}{ll}
\toprule
Setting & Value \\
\midrule
\multicolumn{2}{l}{\textit{Pretraining}}\\
Corpus (size) & ZINC2m \\
Grammar / production rules & RSG / 825 \\
Backbone (Grammar Transformer) & width 512, 4 layers, 8 heads \\
Dropout & 0.1 \\
Max grammar-tree path distance & 3 \\
Atom mask ratio / node loss & 0.25 / scaled cosine error \\
Pretext-task loss weights ($\lambda$) & atom 4.5; fragment 5.0; molecule 1.0; 3D 0.04 \\
Optimiser & AdamW (lr $10^{-3}$, weight decay $10^{-3}$) \\
Batch size & 512 \\
Training steps & 50,000 \\
Gradient clip / EMA / precision & 1.0 / 0.999 / bf16 (AMP) \\
\midrule
\multicolumn{2}{l}{\textit{MoleculeNet transfer}}\\
Tasks & BBBP, BACE, HIV, SIDER, Tox21, ClinTox, ToxCast \\
Split & scaffold, 80/10/10 \\
Protocols & probing; full fine-tuning \\
Shared optimiser & AdamW ($\beta{=}0.9,0.999$) \\
Head lr (both protocols) & $2{\times}10^{-3}$ \\
Encoder lr (full fine-tuning only) & $2{\times}10^{-4}$ ($0.1\times$ head lr) \\
Epochs / patience / batch size & 100 / 10 / 32 \\
EMA decay & 0.9999 \\
\bottomrule
\end{tabular}
\end{table}

\section{Supplementary results and analyses}

\subsection{Extended grammar analysis}

Grammar statistics on ZINC250k~\cite{data:zinc250k} form the primary analysis.
Here we provide complementary results on QM9~\cite{data:qm9} (\cref{SI-fig:qm9_grammar_statistics}), which exhibit the same qualitative trends.
We restrict this cross-method comparison to these two relatively small datasets because MIG and MHG induction are substantially slower on larger corpora, and MHG yields a rapidly inflating rule corpus. The MIG-induced variant is therefore reported as a supplementary grammar-construction ablation, whereas HGR-VAE and HGR-LDF use the RSG-induced grammar unless otherwise stated.

\begin{figure}[!h]
    \centering
    \includegraphics[width=0.6\linewidth]{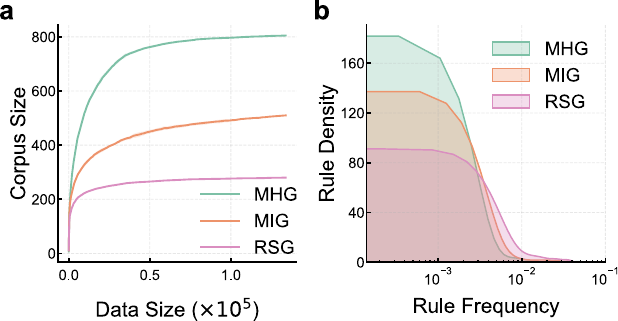}
    \caption{\textbf{Grammar statistics on QM9 dataset.}
{\bf a}, Grammar corpus growth on QM9, showing the number of unique rules discovered as progressively more molecules are included during grammar induction.
{\bf b}, Rule-frequency distributions on QM9.
}
\label{SI-fig:qm9_grammar_statistics}
\end{figure}

\subsection{Extended generation benchmarks}
\label{SI-sec:gen-rel}

The extended tables report the basic, chemistry- and topology-based metrics defined above, including validity without correction, uniqueness, novelty, V.U.N.~\cite{GDSS+ICML2022}, Fr{\'e}chet ChemNet Distance (FCD)~\cite{FCD}, KL score~\cite{guacamol-JCIM2019}, SA distribution~\cite{SAscore-2009}, Neighborhood Subgraph Pairwise Distance Kernel (NSPDK) MMD~\cite{NSPKD-MMD}, SNN, scaffold similarity and curvature-based descriptors~\cite{CurvFilt+NeurIPS2023}.
Arrows indicate the preferred direction of each metric ($\uparrow$: higher values indicate better performance; $\downarrow$: lower values indicate better performance).
Best values for rank-interpretable metrics are shown in bold. Curvature-filtration descriptors are reported as auxiliary metrics and are not highlighted.

Additional MOSES-specific baselines include GraphINVENT~\cite{GraphINVENT-2021}, DisCo~\cite{DisCo-NeurIPS2024} and Cometh~\cite{Cometh-TMLR2025}; the remaining generation baselines are described in Methods.
MHG-VAE~\cite{MHG-VAE-ICML2019} is included only for QM9 and ZINC250k, where direct comparison with a molecular hypergraph grammar baseline is feasible. 
We omit it from larger generation benchmarks because MHG induction is substantially slower and yields a rapidly expanding rule corpus, consistent with the grammar-analysis results.

\begin{table}[!h]
\centering
\setlength{\tabcolsep}{2pt} %
\caption{\textbf{Generation performance on the RingDiv300k benchmark.}}
\resizebox{\linewidth}{!}{
\begin{tabular}{lcccccccccc}
\toprule
Model & Val. $\uparrow$ & Unique. $\uparrow$ & Nov. $\uparrow$ & V.U.N. & KL score $\uparrow$ & FCD $\downarrow$ & NSPDK $\downarrow$ & SA dist. $\downarrow$ & $\kappa_{\text{BF}}\downarrow$ & $\kappa_{\text{OR}}\downarrow$ \\
\midrule
Training 
& 100.0\spm{0.0} & 100.0\spm{0.0} & N/A & N/A & 99.8\spm{0.0} & 0.1737\spm{0.0024} & 0.0001\spm{0.0000} & 0.0123\spm{0.0040} & 0.1312\spm{0.0577} & 0.1203\spm{0.0644} \\
\midrule
DiGress 
& 98.9\spm{0.0} & \textbf{100.0}\spm{0.0} & \textbf{100.0}\spm{0.0} & 98.90\spm{0.04} & 98.2\spm{0.1} & 0.6300\spm{0.0110} & \textbf{0.00041}\spm{0.00006} & 0.4171\spm{0.0247} & 0.1021\spm{0.0316} & 0.2656\spm{0.0582} \\
DeFoG
& 98.3\spm{0.0} & \textbf{100.0}\spm{0.0} & \textbf{100.0}\spm{0.0} & 98.27\spm{0.04} & 98.7\spm{0.1} & 0.8575\spm{0.0087} & 0.00109\spm{0.00002} & 0.1298\spm{0.0081} & 0.4735\spm{0.0373} & 0.4045\spm{0.0088} \\
HoG-Diff 
& \textbf{100.0}\spm{0.0} & 99.9\spm{0.0} & 99.6\spm{0.1} & 99.57\spm{0.1} & 98.2\spm{0.1} & 0.9250\spm{0.0236} & 0.00134\spm{0.00026} & 0.1435\spm{0.0333} & 0.1810\spm{0.0326} & 0.1729\spm{0.03247} \\
SMILES-VAE
& 88.1\spm{0.1} & \textbf{100.0}\spm{0.0} & \textbf{100.0}\spm{0.0} & 88.1\spm{0.1}
& 97.6\spm{0.0} & 0.6642\spm{0.0156} & 0.00072\spm{0.00005} & 0.1653\spm{0.0117} & 0.1556\spm{0.0061} & 0.2375\spm{0.0502} \\
SELFIES-VAE
& \textbf{100.0}\spm{0.0} & \textbf{100.0}\spm{0.0} & \textbf{100.0}\spm{0.0} & \textbf{100.0}\spm{0.0}
& 92.8\spm{0.4} & 2.4266\spm{0.0526} & 0.00286\spm{0.00005} & 0.6077\spm{0.0071} & 0.6891\spm{0.0687} & 0.3874\spm{0.0738} \\
\midrule
HGR-VAE (RSG)
& \textbf{100.0}\spm{0.0} & \textbf{100.0}\spm{0.0} & \textbf{100.0}\spm{0.0} & \textbf{100.0}\spm{0.0}
& 98.1\spm{0.4} & 0.5882\spm{0.0070} & 0.00147\spm{0.00010} & 0.1095\spm{0.0090} & 0.5260\spm{0.1119} & 0.2583\spm{0.0911} \\
HGR-LDF (RSG)
& \textbf{100.0}\spm{0.0} & \textbf{100.0}\spm{0.0} & \textbf{100.0}\spm{0.0} & \textbf{100.0}\spm{0.0}
& \textbf{99.2}\spm{0.1} & \textbf{0.4652}\spm{0.0093} & 0.00106\spm{0.00010} & \textbf{0.0983}\spm{0.0050} & 0.6093\spm{0.1040} & 0.1431\spm{0.0104}\\
\bottomrule
\end{tabular}
}
\label{SI-tab:gen_ringdiv300k}
\end{table}

\begin{table}[!h]
\centering
\setlength{\tabcolsep}{3pt} %
\caption{\textbf{Generation performance on the GuacaMol benchmark.}}
\resizebox{\textwidth}{!}{
\begin{tabular}{lcccccccc}
\toprule
Model & Val. $\uparrow$ & Unique. $\uparrow$ & Nov. $\uparrow$ & V.U.N. & FCD score $\uparrow$ & KL score $\uparrow$ & $\kappa_{\text{BF}}\downarrow$ & $\kappa_{\text{OR}}\downarrow$ \\
\midrule
Training & 100.0\spm{0.0} & 100.0\spm{0.0} & N/A & N/A & 96.12\spm{0.09} & 99.71\spm{0.00} & 0.1276\spm{0.0819} & 0.0573\spm{0.0304}\\
\midrule
DiGress  & 85.2 & \textbf{100.0} & \textbf{99.9} & - & 68.0 & 92.9 & - & - \\ %
Disco    & 86.6 & 86.6 & 86.5 & - & 59.7 & 92.6 & - & - \\
Cometh   & 98.9 & 98.9 & 97.6 & - & 72.7 & 96.7 & - & - \\
DeFoG & 99.0 & 99.0 & - & 97.9 & 73.8 & \textbf{97.7} & - & - \\
HoG-Diff & 99.1\spm{0.1} & \textbf{100.0}\spm{0.0} & 97.0\spm{0.3} & 95.9\spm{0.2} & 78.3\spm{0.3} & 96.9\spm{0.3} & 0.1726\spm{0.0937} & 0.0944\spm{0.0423} \\
SMILES-VAE
& 91.0\spm{0.2} & \textbf{100.0}\spm{0.0} & 99.4\spm{0.0} & 90.6\spm{0.2} & 80.7\spm{0.5} & 92.7\spm{1.0} & 0.3840\spm{0.1309} & 0.2824\spm{0.0890} \\
SELFIES-VAE
& \textbf{100.0}\spm{0.0} & \textbf{100.0}\spm{0.0} & \textbf{99.7}\spm{0.0} & \textbf{99.7}\spm{0.0} & 39.9\spm{0.3} & 88.0\spm{0.6} & 0.5711\spm{0.0620} & 0.2381\spm{0.0315} \\
\midrule
HGR-VAE (RSG)
& \textbf{100.0}\spm{0.0} & \textbf{100.0}\spm{0.0} & 99.4\spm{0.1} & 99.4\spm{0.1} & 83.7\spm{0.3} & 93.8\spm{0.7} & 0.2518\spm{0.0428} & 0.1516\spm{0.0483} \\
HGR-LDF (RSG)
& \textbf{100.0}\spm{0.0} & \textbf{100.0}\spm{0.0} & 98.9\spm{0.0} & 98.9\spm{0.0} & \textbf{84.4}\spm{0.3} & 95.7\spm{0.4} & 0.2974\spm{0.0415} & 0.1151\spm{0.0048} \\
\bottomrule
\end{tabular}%
}
\label{SI-tab:guacamol_results}
\end{table}

\begin{table}[!h]
\centering
\setlength{\tabcolsep}{3pt} %
\caption{\textbf{Generation performance on the QM9 benchmark.}}
\resizebox{\linewidth}{!}{
\begin{tabular}{lcccccccc}
\toprule
Method & Val. $\uparrow$ & Unique. $\uparrow$ & Nov. $\uparrow$ & V.U.N. $\uparrow$ & FCD $\downarrow$ & NSPDK $\downarrow$ & $\kappa_{\text{BF}}\downarrow$ & $\kappa_{\text{OR}}\downarrow$ \\
\midrule
Training & 100.00\spm{0.00} & 99.99\spm{0.00} & N/A & N/A & 0.074\spm{0.004} & 0.0002\spm{0.0000} & 0.1337\spm{0.0560} & 0.2380\spm{0.0762}\\
\midrule
MiCaM    & 99.93 & 93.89 & 83.25 & -- & 1.045 & 0.001 & -- & -- \\
MoFlow   & 91.36\spm{1.23} & 98.65\spm{0.57} & 94.72\spm{0.77} & -- & 4.467\spm{0.595} & 0.017\spm{0.003} & -- & -- \\
EDP-GNN  & 47.52\spm{3.60} & \textbf{99.25}\spm{0.05} & 86.58\spm{1.85} & -- & 2.680\spm{0.221} & 0.005\spm{0.001} & -- & -- \\
GraphEBM & 8.22\spm{2.24} & 97.90\spm{0.05} & \textbf{97.01}\spm{0.17} & -- & 6.143\spm{0.411} & 0.030\spm{0.004} & -- & -- \\
GDSS     
& 95.72\spm{1.94} & 98.46\spm{0.61} & 86.27\spm{2.29} & \textbf{86.48} & 2.900\spm{0.282} & 0.003\spm{0.000} & 0.9252 & 0.6012 \\
DiGress  
& 99.19\spm{0.23} & 96.87\spm{0.24} & 71.69\spm{2.36} & 71.13 & 0.279\spm{0.008} & 0.0006\spm{0.0001} & 0.2511 & 0.3428 \\
DeFoG    & \textbf{100.00}\spm{0.00} & 96.61\spm{0.30} & 72.57\spm{1.89} & - & 0.268\spm{0.006} & 0.0005\spm{0.0001} & 0.1774 & 0.2856 \\
Cometh   & \textbf{100.00}\spm{0.00} & 96.75\spm{0.09} & 72.06\spm{0.19} & - & 0.248\spm{0.012} & 0.0005\spm{0.0001} & 0.2165 & 0.3137 \\
HoG-Diff
& 98.74\spm{0.22}  & 97.14\spm{0.11} & 75.12\spm{0.39} & 71.76\spm{0.19} & 0.172\spm{0.002} & \textbf{0.0003}\spm{0.0000} & 0.1137\spm{0.0322} & 0.2277\spm{0.1214}\\
MHG-VAE
& 99.96\spm{0.03} & 94.41\spm{3.17} & 86.81\spm{1.29} & 81.98\spm{3.98} & 1.534\spm{0.191} & 0.0043\spm{0.0022} & 0.5236\spm{0.2707} & 0.3670\spm{0.2571}\\
\midrule
HGR-VAE (MIG)
& \textbf{100.00}\spm{0.00} & 94.25\spm{0.12} & 71.56\spm{1.06} & 67.45\spm{0.94} & 0.119\spm{0.006} & \textbf{0.0003}\spm{0.0000} & 0.2193\spm{0.0731} & 0.0894\spm{0.0287} \\
HGR-VAE (RSG)
& \textbf{100.00}\spm{0.00} & 96.59\spm{0.27} & 73.15\spm{0.26} & 70.66\spm{0.14} & \textbf{0.110}\spm{0.004} & \textbf{0.0003}\spm{0.0000} & 0.0476\spm{0.0261} & 0.1009\spm{0.0532} \\
\bottomrule
\end{tabular}
}
\label{SI-tab:qm9_results}
\end{table}

\begin{table}[!h]
\centering
\setlength{\tabcolsep}{3pt} %
\caption{\textbf{Generation performance on the ZINC250k benchmark.}}
\resizebox{\linewidth}{!}{
\begin{tabular}{lcccccccc}
\toprule
Method & Val. $\uparrow$ & Unique. $\uparrow$ & Nov. $\uparrow$ & V.U.N. $\uparrow$ & FCD $\downarrow$ & NSPDK $\downarrow$ & $\kappa_{\text{BF}}\downarrow$ & $\kappa_{\text{OR}}\downarrow$ \\
\midrule
Training & 100.00\spm{0.00} & 99.98\spm{0.01} & N/A & N/A & 0.201\spm{0.005} & 0.0001\spm{0.0000} & 0.1007\spm{0.0563} & 0.0478\spm{0.0190}\\
\midrule
MiCaM    & \textbf{100.00} & 88.48 & 99.98 & -- & 32.395 & 0.166 & -- & -- \\
MoFlow   & 63.11\spm{5.17} & 99.99\spm{0.01} & \textbf{100.00}\spm{0.00} & -- & 20.931\spm{0.184} & 0.046\spm{0.002} & -- & -- \\
EDP-GNN  & 82.97\spm{2.73} & 99.79\spm{0.08} & \textbf{100.00}\spm{0.00} & -- & 16.737\spm{1.300} & 0.049\spm{0.006} & -- & -- \\
GraphEBM & 5.29\spm{3.83} & 98.79\spm{0.15} & \textbf{100.00}\spm{0.00} & -- & 35.471\spm{5.331} & 0.212\spm{0.075} & -- & -- \\
GDSS     & 97.01\spm{0.77} & 99.64\spm{0.13} & \textbf{100.00}\spm{0.00} & -- & 14.656\spm{0.680} & 0.019\spm{0.001} & -- & -- \\
DiGress  & 94.99\spm{2.55} & 99.97\spm{0.01} & 99.99\spm{0.01} & -- & 3.482\spm{0.147} & 0.0021\spm{0.0004} & -- & -- \\
DeFoG &
99.22\spm{0.08} & 99.99\spm{0.01} & -- & -- & 1.425\spm{0.022} & 0.0008\spm{0.0001} & -- & -- \\
HoG-Diff &
98.56\spm{0.12} & 99.96\spm{0.02} & 99.53\spm{0.08} & 97.93\spm{0.05} & 1.633\spm{0.017} & 0.0015\spm{0.0016} & 0.1901\spm{0.0695} & 0.0981\spm{0.0411} \\ 
MHG-VAE
& 100.00\spm{0.00} & 99.19\spm{0.02} & \textbf{100.00}\spm{0.01} & 99.18\spm{0.01} & 5.219\spm{0.093} & 0.0123\spm{0.0015} & 4.9820\spm{0.1680} & 2.7055\spm{0.1268}\\
\midrule
HGR-VAE (MIG)
& \textbf{100.00}\spm{0.00} & 99.99\spm{0.01} & 99.98\spm{0.01} & 99.97\spm{0.01} & 0.648\spm{0.011} & 0.0009\spm{0.0000} & 0.3113\spm{0.0846} & 0.1321\spm{0.0410} \\
HGR-VAE (RSG)
& \textbf{100.00}\spm{0.00} & \textbf{100.00}\spm{0.00} & 99.99\spm{0.00} & \textbf{99.99}\spm{0.00} & 0.726\spm{0.017} & 0.0006\spm{0.0000} & 0.1894\spm{0.0680} & 0.0711\spm{0.0614} \\
HGR-LDF (RSG)
& \textbf{100.00}\spm{0.00} & \textbf{100.00}\spm{0.00} & 99.98\spm{0.01} & 99.98\spm{0.01} & \textbf{0.595}\spm{0.012} & \textbf{0.0005}\spm{0.0000} & 0.1316\spm{0.0481} & 0.0636\spm{0.0283} \\
\bottomrule
\end{tabular}
}
\label{SI-tab:zinc_results}
\end{table}

\begin{table}[!h]
\centering
\setlength{\tabcolsep}{2pt} %
\caption{\textbf{Generation performance on the MOSES benchmark.}}
\resizebox{\textwidth}{!}{
\begin{tabular}{lcccccccccc}
\toprule
Method & Val. $\uparrow$ & Unique. $\uparrow$ & Nov. $\uparrow$ & V.U.N. $\uparrow$ & FCD $\downarrow$ & Filters $\uparrow$ & SNN $\uparrow$ & Scaf. $\uparrow$ & $\kappa_{\text{BF}}\downarrow$
& $\kappa_{\text{OR}}\downarrow$ \\
\midrule
Training & 100.0\spm{0.0} & 100.0\spm{0.0} & - & - & 0.518\spm{0.007} & 100.0\spm{0.1} & 0.586\spm{0.000}& - & 0.1138\spm{0.0303} & 0.1365\spm{0.01592}\\
\midrule
GraphINVENT & 96.4 & 99.8 & - & - & 1.22 & 95.0 & 0.54 & 12.7 & - & - \\
DiGress     & 85.7 & \textbf{100.0} & 95.0 & - & 1.19 & 97.1 & 0.52 & 14.8 & - & - \\
DisCo      & 88.3 & \textbf{100.0} & \textbf{97.7} & - & 1.44 & 95.6 & 0.50 & 15.1 & - & - \\
Cometh       & 90.5 & 99.9 & 92.6 & - & 1.27 & \textbf{99.1} & 0.54 & \textbf{16.0} & - & - \\
DeFoG   & 92.8 & 99.9 & 92.1 & - & 1.95 & 98.9 & \textbf{0.55} & 14.4 & - & - \\
HoG-Diff
& 99.6\spm{0.1} & 99.9\spm{0.0} & 89.1\spm{0.1} & 88.7\spm{0.1} & 0.937\spm{0.0125} & 96.6\spm{0.2} & 0.519\spm{0.001} & 8.0\spm{0.4} & 0.1365\spm{0.0196} & 0.1625\spm{0.0200} \\
\midrule
HGR-VAE (RSG)             & \textbf{100.0}\spm{0.0} & \textbf{100.0}\spm{0.0} & 95.5\spm{0.1} & \textbf{95.5}\spm{0.1} & 0.676\spm{0.0078} & 97.8\spm{0.1} & 0.519\spm{0.000} & 11.4\spm{0.5} & 0.2860\spm{0.0624} & 0.1348\spm{0.0286} \\
HGR-LDF (RSG)      & \textbf{100.0}\spm{0.0} & \textbf{100.0}\spm{0.0} & 95.2\spm{0.1} & 95.2\spm{0.1} & \textbf{0.626}\spm{0.0087} & 98.0\spm{0.1} & 0.523\spm{0.001} & 10.5\spm{0.3} & 0.2587\spm{0.0329} & 0.0761\spm{0.0141} \\
\bottomrule
\end{tabular}%
}
\label{SI-tab:moses_results}
\end{table}

\subsection{HGR-FM transfer benchmarks}

The complete task-level transfer results for the pretrained molecular encoders are reported in \cref{SI-tab:FM_transfer_combined}.
This table complements the normalised summary in Fig.~5\textbf{b} by showing the original per-task AUC values and across-run standard deviations under both probing and full fine-tuning.

\begin{table*}[htbp]
\centering
\begin{threeparttable}
\caption{\textbf{Transfer performance of pretrained molecular encoders.}}
\label{SI-tab:FM_transfer_combined}
\begin{tabular}{@{}c@{}}
\begin{tabular}{l c c c c c c c c}
\toprule
Method & BBBP & Tox21 & ToxCast & SIDER & ClinTox & HIV & BACE & avg $\uparrow$ \\
\midrule
\multicolumn{9}{l}{\textit{a. Probing with frozen encoder}} \\
\midrule
Infomax & 60.8\spm{0.38} & 67.0\spm{0.40} & 58.3\spm{0.24} & 58.2\spm{0.76} & 62.6\spm{0.28} & 71.3\spm{1.36} & 65.1\spm{0.67} & 63.33 \\
EdgePred & 52.7\spm{1.45} & 63.0\spm{0.77} & 54.1\spm{0.39} & 51.7\spm{0.99} & 48.2\spm{5.46} & 65.2\spm{1.06} & 58.6\spm{1.76} & 56.21 \\
AttrMasking & 51.8\spm{0.23} & 69.3\spm{0.04} & 57.7\spm{0.08} & 51.3\spm{0.09} & 54.5\spm{0.44} & 60.5\spm{0.31} & 61.8\spm{0.62} & 58.13 \\
ContextPred & 58.8\spm{0.57} & 68.3\spm{0.24} & 58.8\spm{0.36} & 59.2\spm{0.19} & 40.0\spm{1.51} & 67.0\spm{0.62} & 59.6\spm{2.53} & 58.81 \\
GraphCL & 63.0\spm{0.29} & 67.6\spm{0.39} & 57.4\spm{0.48} & 52.8\spm{0.95} & 54.7\spm{5.10} & 64.8\spm{1.69} & 66.3\spm{0.29} & 60.94 \\
GraphLoG & 54.5\spm{0.30} & 66.8\spm{0.17} & 57.4\spm{0.17} & 58.0\spm{0.58} & 57.6\spm{1.29} & 65.2\spm{0.54} & 72.4\spm{0.65} & 61.70 \\
GraphMAE & 56.5\spm{0.41} & 66.7\spm{0.45} & 57.6\spm{0.10} & 52.0\spm{0.82} & 44.3\spm{0.57} & 60.5\spm{0.54} & 61.8\spm{5.53} & 57.06 \\
GraphMVP & 57.9\spm{0.68} & 66.9\spm{0.17} & 58.5\spm{0.09} & 56.0\spm{1.16} & 42.7\spm{1.99} & 67.1\spm{0.83} & 65.3\spm{0.20} & 59.20 \\
MORE & 67.9\spm{0.28} & 70.6\spm{0.33} & 62.2\spm{0.16} & 60.7\spm{0.28} & 63.5\spm{1.26} & 72.8\spm{0.32} & 80.0\spm{0.62} & 68.24 \\
\textbf{HGR-FM}
& \textbf{74.4}\spm{0.96}
& \textbf{78.5}\spm{0.57}
& \textbf{68.9}\spm{0.11}
& \textbf{63.0}\spm{0.58}
& \textbf{89.7}\spm{1.59}
& \textbf{76.9}\spm{1.02}
& \textbf{84.1}\spm{0.56}
& \textbf{76.51} \\
\midrule
\multicolumn{9}{l}{\textit{b. Full fine-tuning}} \\
\midrule
Infomax & 68.0\spm{1.05} & 75.0\spm{0.55} & 62.4\spm{0.57} & 59.7\spm{0.47} & 71.1\spm{3.88} & 76.8\spm{1.38} & 76.8\spm{1.32} & 69.97 \\
EdgePred & 65.3\spm{1.96} & 76.3\spm{0.26} & 63.6\spm{0.28} & 61.3\spm{0.54} & 64.5\spm{2.67} & 75.7\spm{1.05} & 79.9\spm{1.27} & 69.51 \\
AttrMasking & 63.4\spm{1.45} & 76.4\spm{0.19} & 63.4\spm{0.52} & 60.0\spm{0.87} & 71.1\spm{2.46} & 76.3\spm{0.28} & 79.7\spm{0.47} & 70.04 \\
ContextPred & 67.1\spm{0.81} & 74.4\spm{0.12} & 63.7\spm{0.12} & 60.9\spm{0.78} & 59.0\spm{1.99} & 76.6\spm{0.45} & 79.1\spm{2.13} & 68.69 \\
GraphCL & 68.0\spm{2.05} & 74.5\spm{0.12} & 62.4\spm{0.59} & 59.2\spm{1.12} & 75.3\spm{3.67} & 76.3\spm{1.09} & 76.8\spm{0.36} & 70.36 \\
GraphLoG & 66.8\spm{2.55} & 74.7\spm{0.25} & 62.4\spm{0.55} & 59.6\spm{0.67} & 64.2\spm{1.27} & 76.6\spm{0.77} & 82.3\spm{0.42} & 69.51 \\
GraphMAE & 68.0\spm{2.54} & 75.6\spm{0.27} & 63.4\spm{0.14} & 60.2\spm{0.49} & 70.8\spm{4.15} & 76.6\spm{0.87} & 82.1\spm{1.40} & 70.96 \\
GraphMVP & 71.3\spm{1.13} & 75.0\spm{0.37} & 63.4\spm{0.26} & \textbf{62.9}\spm{0.29} & 68.2\spm{5.95} & 75.8\spm{1.44} & 78.8\spm{4.61} & 70.77 \\
MORE & 71.9\spm{0.94} & 75.6\spm{0.54} & 64.6\spm{0.58} & 60.9\spm{0.62} & 81.0\spm{0.65} & 77.0\spm{0.74} & 82.8\spm{1.33} & 73.40 \\
\textbf{HGR-FM}
& \textbf{74.3}\spm{1.08}
& \textbf{79.0}\spm{0.41}
& \textbf{70.8}\spm{0.12}
& 62.1\spm{1.05}
& \textbf{89.4}\spm{1.00}
& \textbf{77.5}\spm{0.21}
& \textbf{84.0}\spm{0.43}
& \textbf{76.73} \\
\bottomrule
\end{tabular}%
\end{tabular}
\begin{tablenotes}
\footnotesize
\item All values are AUC (mean $\pm$ s.d.) over three independent runs using scaffold splits. In probing, the pretrained encoder is frozen and only probe layers are trained; in full fine-tuning, the encoder and task-specific head are optimised jointly.
\end{tablenotes}
\end{threeparttable}
\end{table*}

\subsection{Computational efficiency analysis}

All experiments in this work, including the runtime measurements reported below, were performed with PyTorch on a single NVIDIA L40S GPU (46 GB memory) and an AMD EPYC 9374F 32-core processor.

We analyse computational cost from two complementary perspectives: \textit{(i)} evaluation-time overhead, which can dominate large-scale benchmarking when many metrics are computed across datasets and random seeds; and \textit{(ii)} model-side runtime, which includes training and inference and determines practical iteration speed and sample throughput.

\subsubsection[Evaluation optimisation]{Evaluation optimisation}

To improve the scalability of our evaluation pipeline, we implemented targeted optimizations for the most time-consuming metrics. As shown in \cref{SI-fig:index_time_opt}\textbf{a}, the NSPDK computation is accelerated by approximately 5-fold. We also obtain substantial speedups for the curvature-filtration descriptors ($\kappa_{OR}$ and $\kappa_{BF}$), with nearly 2-fold acceleration.

\cref{SI-fig:index_time_opt}\textbf{b} shows that these gains are consistent across additional standard benchmarks (QM9, ZINC250k, MOSES, and GuacaMol), indicating that the optimised implementations reduce evaluation overhead across several data regimes. Overall, these optimisations increase evaluation throughput while preserving metric informativeness, enabling faster and more reproducible assessment of de novo generation performance.

\begin{figure}[h]
    \centering
    \includegraphics[width=\linewidth]{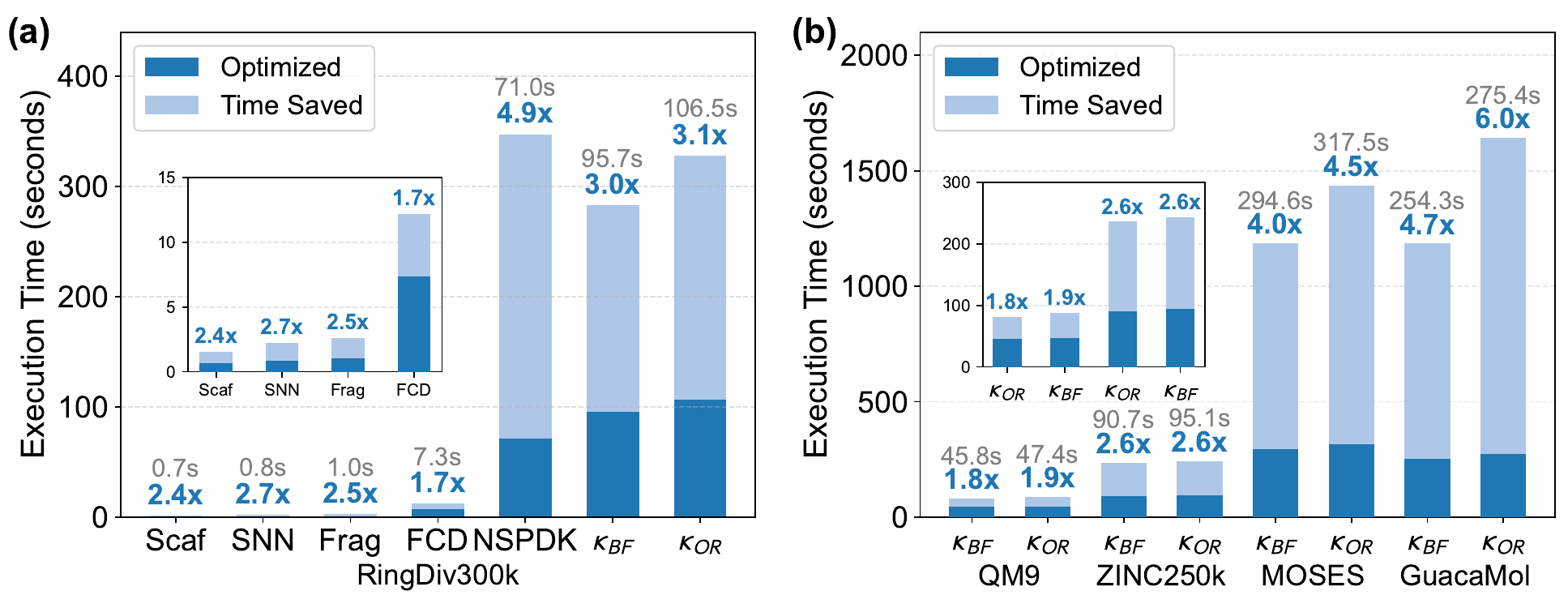}
    \caption{
{\bf Accelerated evaluation of generation metrics.}
{\bf a}, Wall-clock execution time of representative evaluation metrics on the RingDiv300k benchmark (Scaf, SNN, Frag, FCD, NSPDK, $\kappa_{BF}$, $\kappa_{OR}$). Dark bars indicate the optimised implementation; light shading denotes time saved relative to the original pipeline.
Blue bold numbers above bars denote the speed-up factor, and grey numbers report the post-optimisation runtime.
{\bf b}, Runtime of the two slowest metrics ($\kappa_{BF}$ and $\kappa_{OR}$) across additional standard datasets (QM9, ZINC250k, MOSES and GuacaMol), highlighting consistent acceleration under diverse data regimes.
Data represent mean values for ten independent replicates.}
    \label{SI-fig:index_time_opt}
\end{figure}

\subsubsection{Training and inference time}

We further report training and inference runtime to characterise the practical cost of model optimisation and sample generation. Together with the evaluation-time acceleration above, these results provide a system-level view of computational efficiency.

\cref{SI-tab:gen_runtime} summarises end-to-end training time and inference time across diverse molecular datasets, reflecting the costs of model development and sample generation at the same evaluation scale.
Across the settings reported here, both training and generation show favourable runtime relative to the compared baselines.

\begin{table}[!th]
\centering
\begin{threeparttable}
\caption{Training and inference runtime across molecular datasets.}
\label{SI-tab:gen_runtime}
\small
\setlength{\tabcolsep}{4.5pt}
\renewcommand{\arraystretch}{1.05}
\begin{tabular}{lccccc}
\toprule
Method & QM9 & ZINC250k & MOSES & GuacaMol & RingDiv300k \\
\midrule
\multicolumn{6}{l}{\textit{Training time (h)}} \\
\midrule
DiGress & - & - & - & - & 63.0\\
DeFoG & - & - & - & - & 53.1\\
HoG-Diff & 3.1 & 6.5  & 7.8 & 30.0 & 8.1 \\
HGR-VAE (RSG) & 0.3 & 3.1  & 1.4  & 8.4 & 3.0 \\
\midrule
\multicolumn{6}{l}{\textit{Inference time (min)}} \\
\midrule
DiGress & - & - & 155.5\spm{2.0} & 267.3\spm{2.9} & 166.5\spm{2.2} \\
DeFoG & - & - & - & - & 221.4\spm{1.1} \\
HoG-Diff & 12.6\spm{0.2} & 48.8\spm{0.4}  & 11.0\spm{0.8} & 93.7\spm{1.7} & 28.4\spm{0.6} \\
HGR-VAE (RSG) & 1.2\spm{0.1} & 2.5\spm{0.0}  & 3.4\spm{0.1}  & 2.1\spm{0.0} & 2.5\spm{0.2} \\
\bottomrule
\end{tabular}
\begin{tablenotes}
\footnotesize
\item Training time is reported in hours and inference time in minutes. Inference values are mean $\pm$ s.d. A dash indicates that comparable runtime was unavailable.
\end{tablenotes}
\end{threeparttable}
\end{table}

\newpage
\bibliography{ref}

\end{document}